\documentclass{article}

\usepackage[preprint]{neurips_2026}

\usepackage[utf8]{inputenc} 
\usepackage[T1]{fontenc}    
\usepackage{hyperref}       
\usepackage{url}            
\usepackage{booktabs}       
\usepackage{amsfonts}       
\usepackage{nicefrac}       
\usepackage{microtype}      
\usepackage{xcolor}         
\usepackage{wrapfig}

\usepackage{microtype}
\usepackage{graphicx}
\usepackage{subcaption}
\usepackage{booktabs}
\usepackage{amsmath, amssymb, mathtools, amsthm}
\usepackage{tabularx, multirow}
\usepackage[labelfont=bf,justification=centering]{caption}
\usepackage{pdfpages}
\usepackage{soul}
\usepackage[titletoc]{appendix}
\usepackage{chngcntr}
\counterwithin{subsection}{section}

\usepackage{color,soul}
\usepackage[colorinlistoftodos]{todonotes}

\usepackage[capitalize,noabbrev]{cleveref}

\theoremstyle{plain}

\newtheorem{proposition}{Proposition}
\newtheorem{assumption}{Assumption}
\crefname{assumption}{Assumption}{Assumptions}

\title{Artificial Structure Function Search: Preserving Artificial Functional Connectivity for Structured Pruning}

\author{
  Mindula Illeperuma$^{1}$\thanks{Corresponding author} \qquad
  Rafael Pina$^{1}$ \qquad
  Charuka Herath$^{1}$ \qquad \\
  \textbf{Sharmarke A. Gabayre$^{1}$} \qquad
  \textbf{Varuna De Silva$^{2}$} \\[4pt]
  $^{1}$Intitute for Digital Technologies, Loughborough University London \\[2pt]
  $^{2}$Center for Neuromorphic Intelligence, Aston University \\[4pt]
  \texttt{\{k.m.illeperuma,r.m.pina,c.k.herath,s.gabayre\}@lboro.ac.uk} \\
  \texttt{v.desilva@aston.ac.uk}
}

\begin{document}

\maketitle

\begin{abstract}
Structured pruning is a  model compression technique that is used to reduce the computational cost of deploying deep neural networks on resource-constrained devices. Popular methods of pruning rely on opaque heuristics or weight-based criteria that give no indication as to the structural dependencies in the network. To address these limitations we present Artificial Structure Function Search (ASF-S): a novel structured pruning framework. ASF-S utilizes Principle Gradient Importance (PGI): a novel prune-candidate selection criteria that is inspired by structure-function relationships in the brain. By ensuring the pruned structure of the model respects topographical organization of the output layer, we define Artificial Functional Connectivity (AFC) for artificial neural networks. AFC provides evidence to demonstrate that accurate smaller networks can be found using careful prune candidate selection criteria. We present results for PGI as a selection criterion and for ASF-S as a pruning framework against recent benchmarks, demonstrating that our method yields model variants with 70\% parameter reduction, that can recover baseline accuracy without re-training the pruned layers. 
\end{abstract}

\section{Introduction}
The success of Deep Neural Networks (DNNs) in various domains, such as image recognition and natural language processing \cite{tian2024visual,simonyan2014very,devlin2019bert}, has emerged as a cornerstone of Artificial Intelligence (AI). Achieving top performance of these networks relies on the availability of significant resources, including advanced computational hardware and sufficient memory capacity \cite{han2015deep}. For example, popular commercial models such as GPT-3 and GPT-4 include over 175 billion parameters, and increasing parameter size leads to high inference costs \cite{cheng2023survey}. 

Unfortunately, this limits the deployment of DNNs on edge devices with constrained computational resources, such as smartphones, which are restricted by memory, energy, and CPU/GPU capacity \cite{demir2025overview,han2015deep,you2019gate,frantar2023sparsegpt}. Additionally, to see success of DNNs in real-world applications, such as autonomous driving, patient triage in healthcare and industrial robotics, their complexity and memory footprint need to be reduced \cite{demir2025overview,han2015deep}. Similarly, another problem with large models is any redundant parameters leave them more susceptible to adversarial attacks \cite{kraidia2024defense,sehwag2020hydra}. This makes it difficult when attempting to deploy in sectors using private and sensitive data such as in healthcare even when they could significantly benefit from DNN based system integrations. 
\begin{figure*}[ht] 
    \centering
    \includegraphics[width=0.8\textwidth,keepaspectratio]{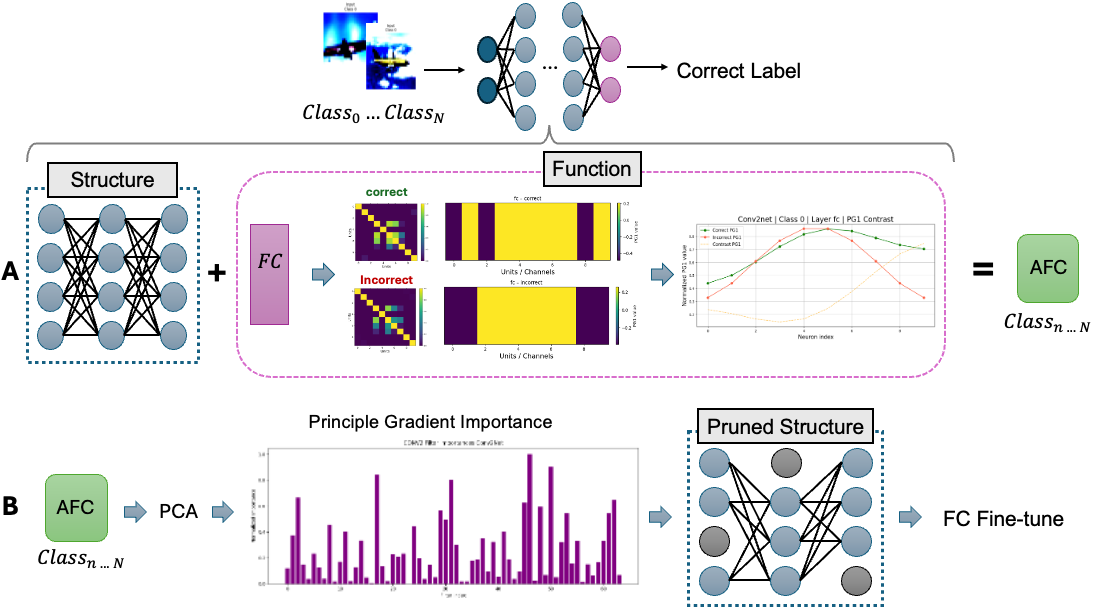}
    \caption{The overall framework of our proposed Artificial Structure Function Search (ASF-S). Taking a model that has been trained in an image classification task, we propose \textbf{A}: AFC as artificial functional connectivity; the relationship between all components in the original trained network (Structure) and the topographical organization (Function) of the output layer. \textbf{B}: As PG1 is collected per class in the dataset, we compute PCA to get the top components representing the class-wise differences in AFC and use this to compute the principle gradient importance (PGI) scores per layer. We aim to yield a sparse network by preserving the AFC required for the task, despite changes to structure caused by pruning. We demonstrate relevance of preserving AFC by freezing the pruned layers and only fine-tuning the output layer to recover model performance.}
    \label{fig:ASF-S}
\end{figure*}

To address the challenges of deploying large neural networks in real-world settings, a variety of model compression techniques such as quantization \cite{shao2023omniquant} and knowledge distillation \cite{gu2023minillm} have been developed. Among these, pruning has emerged as a key strategy, selectively removing weights, neurons, or filters to yield a sparse and efficient network \cite{ashkboos2024slicegpt,ma2023llm}. Pruning methods can be broadly categorized into unstructured approaches, which remove individual parameters based on magnitude or importance \cite{cheng2023survey}, and structured approaches, which remove entire components such as filters or neurones \cite{li2021npas,he2023structured}. Depending on when pruning is applied, methods can also be classified as pre-training, during-training (progressive), or post-training with subsequent fine-tuning \cite{morcos2019one}.  

These methods are all centered around finding \textit{winning tickets}, that is the concept proposed by \cite{frankle2019lottery}, that randomly initialized dense neural networks contain a sparse subnetwork that can match accuracy of the dense network when trained. However finding this sparse subnetwork efficiently using post-training structured methods requires carefully designed scoring criteria \cite{cheng2023survey}. 
Methods tend to yield better results when structured pruning is applied iteratively as one-shot methods damage internal representations of the dense network, making accuracy recovery difficult \cite{janusz2025one,cheng2023survey}. Additionally, after structured pruning it is generally considered essential to retrain the full model due to representational changes in the pruned layers \cite{blalock2020state}. This indicates that, despite having pruning criterion, structured pruning techniques lack a way to identify the accurate smaller networks that includes the network components that have learned the essential feature representations for whatever task the model has been trained on.

In neuroscience, brain networks are usually investigated from a task-dependent perspective \cite{zhao2023task}. To then learn the relationship between anatomical structure of the brain  and the emergent functional activations during tasks, network neuroscientists look into structure-function coupling (SFC). SFC gives us insight into the mechanistic relationship between the structure of networks and how those networks co-activate in response to tasks \cite{fotiadis2024structure}. From this research we can learn about the topological organization of the brain that facilitates complex intermediate processing that enables cognition and behavior in humans \cite{fotiadis2024structure,mesulam1998sensation,margulies2016situating}. In complex organisms, large-scale topological organization across the brain is captured by the unimodal-transmodal axis; a principal gradient of cortical organization that goes from primary sensory regions (unimodal) to association regions (transmodal) such as the default mode network (DMN) \cite{margulies2016situating,ronde2024default}. Specific to visual, object recognition tasks, the brain demonstrates connectivity between primary visual processing areas in the visual cortex, and memory association regions that make up the DMN \cite{paquola2025architecture,zhang2018intrinsic}. This highlights the importance of preserving structure-function relationships in humans to enable object recognition based on learned representations.  

On the other hand, looking into functional networks in artificial neural networks (ANNs) is an emerging area stating that analysis of the topology of ANNs can give explanations for understanding their behavior \cite{zhang2023functional,songdechakraiwut2025functional}. Additionally, Convolutional Neural Networks (CNNs) are inspired by the early visual cortex (comprising of unimodal regions) that make up the primary visual cortex \cite{fukushima1983neocognitron,lecun2002gradient,zeiler2014visualizing,yamins2014performance}. However, their learned representations are not explicitly interpretable, making it difficult to identify filters that encode task-critical information \cite{zhang2023functional,bau2017network}. Consequently, structured pruning that isn't structure-function aware can disrupt the learned feature representations, requiring full-model retraining to recover.  

In machine learning, the architectural design of a neural network already gives a pre-determined structural topological organization (in neuroscience-analogous terms of structure). Given the importance of structure-function relationships in preserving brain functionality following ablation of white matter tracts \cite{uddin2013complex,uddin2008residual}, we hypothesize that preservation of the structure-function relationship in ANNs during pruning can yield a sparse network that will not require fine-tuning of the pruned layers. 
In this work, we introduce Artificial Structure Function Search (ASF-S), as shown in Figure \ref{fig:ASF-S}: a structured pruning framework inspired by neuroscience, that considers the structure-function relationship in ANNs to guide pruning decisions.  

Our contributions can be summarized as follows:  
\begin{itemize} 
    \item \textbf{Artificial Functional Connectivity (AFC):} The novel term we introduce to describe the minimal structure-function relationship that must be preserved during structured pruning. Inspired by a neuroscience brain analysis technique, we believe that, so long as functional connectivity in the output layer is preserved, we can isolate the parts of the network that are safe to remove, thus enabling accuracy recovery without fine-tuning the pruned layers. Our results show that pruning with respect to AFC, yields a sparse network that retains the components in each layer that are necessary for the learned task. 
    
    \item \textbf{Principal Gradient Importance Score (PGI):} We introduce a biologically inspired saliency metric derived from using AFC of the output layer. The PGI score quantifies the contribution of units in preceding layers, to the topological organization of the output layer, when the model correctly classifies an image, enabling us to prune neurons or filters in earlier layers, that do not contribute to the functional connectivity of the output layer. 

    \item \textbf{Artificial Structure Function-Search (ASF-S):} We propose a post-training, functionally guided pruning framework that leverages PGI scores to preserve the sparse subnetwork that maintains AFC, while reducing FLOPs and parameters. Fine-tuning can be done on the whole model consistent with current literature, but our most prominent finding is that only re-training the output layer is sufficient to maintain  100\% of baseline model performance at sparsity of up to 70\% in a CNN architecture. We also provide pruning results from a transformer based architecture to demonstrate usability in that domain. 
\end{itemize}

\section{Related Works}
\textbf{Finding Sparse Networks Using Pruning.}
The Lottery Ticket Hypothesis (LTH) \cite{frankle2019lottery} posits that sparse subnetworks within dense networks, when trained in isolation from initialization, can match full-model performance. Neural network pruning has thus become a popular approach for identifying these \textit{winning tickets}, reducing ANN over-parameterization and computational cost. Pruning methods are broadly classified as \textit{unstructured} or \textit{structured} \cite{cheng2023survey}. Unstructured pruning removes individual weights deemed less important \cite{lee2018snip,lecun1989optimal,wang2020grasp,frankle2019lottery}; for example, HRank \cite{lin2020hrank} ranks filters or units using local activation characteristics alone. However, such methods ignore redundancy, where highly ranked components may capture similar features, reducing representation diversity after pruning \cite{zniyed2025singular}. Moreover, irregular sparse weight patterns limit computational speedup and FLOPs reduction on standard hardware \cite{guo2024interpretable,han2016eie,zhang2021hardware}. In contrast, structured pruning removes entire units (e.g. filters or neurons), reducing both parameter count and computation while producing regular sparse matrices better suited to resource-constrained deployment \cite{he2017channel,luo2017thinet,lecun1989optimal,lin2020hrank,zniyed2025singular,guo2024interpretable}.

\textbf{Pruning Using Spectral Graphs.}
Neural network pruning usually considers ANNs from a probabalistic or gradient flow perspective \cite{malach2020proving,evci2020rigging}. However considering ANNs from a spectral graph theory perspective has proven very useful as it enables the understanding of deeper networks by considering graph properties \cite{laenen2023one}. Some works also use graphical convolutional network (GCNs) to capture topological structure of ANNs and identify pruning candidates \cite{liu2025structure}.
In particular understanding network connectivity using eigenvalues \cite{alon1986eigenvalues} and using this to prune, enables layerwise pruning in one shot following training \cite{laenen2023one}. Thus demonstrating the benefit of topology-aware pruning strategies.

\textbf{Principle Gradient Extraction Using Diffusion Map Embedding.}
Brain regions exhibit diverse structures and connectivity profiles, making raw functional interaction measurements noisy and difficult to interpret \cite{oliver2019quantifying}. Diffusion map embedding addresses this by applying nonlinear dimensionality reduction to functional connectivity matrices, capturing relationships that linear methods may miss \cite{Coifman2006_DiffusionMaps}. From these embeddings, the first principal gradient (PG1) emerges as a continuous low-dimensional axis representing the dominant mode of variation in functional connectivity across regions. Regions with similar PG1 values share connectivity patterns and structural characteristics, indicating membership in the same large-scale organizational network (e.g. primary visual cortex (unimodal) or default mode network (transmodal)) \cite{margulies2016situating}. PG1 therefore provides a topographical view of network organization \cite{Yeo2011_7Networks}, aiding interpretation of large-scale cortical hierarchies and functional network arrangement during tasks or following structural injury \cite{smallwood2021default}. A simplified example of how principal gradient values reflect cortical connectivity patterns is provided in Section 3 of Figure~\ref{fig:pg_ext_brain} in section C of the Appendix.

\textbf{Using Principle Gradient Vectors For ASF-S.}
Diffusion map (DM) \cite{Coifman2006_DiffusionMaps} techniques have been widely applied in computer vision \cite{pachauri2014permutation} and textual network embeddings \cite{zhang2018diffusion}. Beyond these uses, DM provides a principled way to uncover the intrinsic geometry of representations by capturing relationships across units, not just across stimulus conditions. Recent work has shown that neural networks, including large language models, organize internal activations along low-dimensional manifolds that reflect task-specific distinctions \cite{valeriani2023geometry}. By constructing diffusion kernels over unit-wise activations taken across samples, one can reveal coherent geometric structures (e.g., clusters, axes, and manifolds) that encode functionally meaningful variation (or functional connectivity) within the model. In the context of pruning, conventional strategies often focus on magnitude or gradient-based criteria but rarely consider how units jointly structure information. Here we determine AFC by computing PG1 for the output layer, using a contrast vector (PG1 for correct minus PG1 for incorrect) to guide pruning. Thus identifying the components in each layer that must be maintained in the structure to preserve AFC. We outline this in \textbf{A} of Figure~\ref{fig:ASF-S}.

\section{Preliminaries}\label{sec:pre}
\textbf{Artificial Functional Connectivity Maps.}  
In neuroscience, \textit{functional connectivity} (FC) characterizes statistical dependencies between the activity of distinct brain regions, typically computed by correlating time series of different brain regions \cite{fox2007spontaneous,cohen2008defining,van2010exploring}. Drawing an analogy to neural networks, each unit in a layer $\ell$ produces an activation vector across samples, which can be treated as the analogue of a neural time series. Let $N_\ell$ denote the number of units in layer $\ell$, and $S$ the number of samples. We define the activation matrix
\begin{equation}
\mathbf{X}^{(\ell)} = 
\big[\,\mathbf{x}^{(\ell)}_{1}, \dots, \mathbf{x}^{(\ell)}_{N_\ell}\,\big]^\top 
\in \mathbb{R}^{N_\ell \times S},
\end{equation}
where $\mathbf{x}^{(\ell)}_i \in \mathbb{R}^{S}$ contains the activations of unit $i$ across samples.

Pairwise functional similarity between units is quantified using cosine similarity:
\begin{equation}
A_{ij} = 
\frac{\mathbf{x}^{(\ell)}_i{}^\top \mathbf{x}^{(\ell)}_j}{
\|\mathbf{x}^{(\ell)}_i\|_2 \,\|\mathbf{x}^{(\ell)}_j\|_2 }.
\end{equation}
Although Pearson correlation is the most direct analogue to FC in neuroscience, cosine similarity is scale-invariant, robust for non-negative activations (e.g., ReLU layers), and aligns with common practice in representational similarity analysis (RSA) when deriving representational dissimilarity matrices (RSD) \cite{kriegeskorte2008representational,kornblith2019similarity}.

To prepare these similarities for spectral analysis, we transform them into a weighted adjacency matrix:
\begin{equation}
W_{ij} = \exp\Big[-\frac{(1-A_{ij})^2}{2\sigma^2}\Big], \quad 
\mathbf{W} \leftarrow \frac{1}{2}(\mathbf{W} + \mathbf{W}^\top) + \eta I,
\end{equation}
where $\eta$ ensures numerical stability. This matrix $\mathbf{W}$ represents the network-level functional connectivity. A few examples of our AFC matrices for different classes can be seen in Figure~\ref{fig:fc_maps}.

\textbf{The Lottery Ticket Hypothesis (LTH).}  
The LTH posits that a randomly-initialized dense neural network $f(x;\theta_0)$, with parameters $\theta_0 \sim \mathcal{D}$, contains a sparse subnetwork that can match its test performance after retraining \cite{frankle2019lottery}. Formally, there exists a binary mask $m^\star \in \{0,1\}^d$ such that
\begin{equation}
    f(x; m^\star \odot \theta_0) \approx f(x; \theta_0) \quad \text{after training from } \theta_0,
\end{equation}
where $\odot$ denotes element-wise multiplication. While LTH guarantees the existence of $m^\star$, standard pruning methods may not reliably recover it. Theoretical support exists in restricted settings such as linear networks, overparameterized models, and kernel approximations \citep{malach2020proving,pensia2020optimal}, and empirically, winning tickets generalize across datasets and architectures \citep{chen2020lottery,yao2023probabilistic,cheng2023survey}.


\section{Artificial Structure Function Search}
Here we propose Artificial Structure Function Search (ASF-S): a method of extracting \textit{"Function"} for the different model architectures and datasets. 

\textbf{Extracting Network Function By Principal Gradient Extraction.}  
To capture nonlinear structure in these connectivity patterns, we employ diffusion map embedding \cite{Coifman2006_DiffusionMaps}, a spectral technique that extracts principal gradients reflecting dominant modes of coordinated activity across units. Let the degree matrix be
\begin{equation}
\mathbf{D} = \mathrm{diag}(\mathbf{W}\mathbf{1}),
\end{equation}
and the normalized diffusion operator
\begin{equation}
\mathbf{K} = \mathbf{D}^{-\alpha} \mathbf{W} \mathbf{D}^{-\alpha}, \quad \alpha \in [0,1].
\end{equation}

The eigen-decomposition of $\mathbf{K}$,
\begin{equation}
\mathbf{K}\mathbf{v}_k = \lambda_k \mathbf{v}_k, \quad \lambda_1 \ge \lambda_2 \ge \dots \ge \lambda_{N_\ell},
\end{equation}
yields principal gradients \cite{margulies2016situating}. The first non-trivial eigenvector $\mathbf{v}_2$ captures the dominant mode of co-variation among units:
\begin{equation}
\mathbf{g}^{(\ell)} = \mathbf{v}_2,
\end{equation}
normalized to unit length. Differences between correct and incorrect trials highlight shifts in layer-level coordination.

The principle gradient (PG1) vector of the output layer forms the foundation for extracting PGI scores which drive our ASF-S method (a few examples of our PG1 vectors can be seen in Figure~\ref{fig:pg_vectors}). Our structured pruning method preserves units that maintain the dominant functional interactions driving correct classification.

We leverage the PG1 collected from the output layer to compute 
\textit{Principal Gradient Importance} (PGI) scores that rank unit contribution of the preceeding layers, to the PG1 contrast vector obtained by calculating the difference between PG1 of correct and incorrect classification. Our procedure follows three stages:

\textbf{Class-wise Gradient Contrast.}
For each class $c$, we compute the difference between the PG1 vectors 
obtained from correctly and incorrectly classified samples:
\begin{equation}
    \Delta g_c = g^{\mathrm{corr}}_c - g^{\mathrm{incorr}}_c .
\end{equation}
Stacking these across all classes yields the contrast matrix
\begin{equation}\label{eq:cont_matrix}
    \Delta G =
    \begin{bmatrix}
        \Delta g_1^\top \\
        \vdots \\
        \Delta g_C^\top
    \end{bmatrix}.
\end{equation}

\textbf{PCA to Obtain a Class-General Discriminative Direction.}
As we want to ensure AFC resembles the structure-function relationships in the network required to classify all classes it has been trained on, we apply Principal Component Analysis 
(PCA) \cite{guo2025slimllm,Greenacre2022_PCA} to $\Delta G$ and extract the 
top $K$ components
\begin{equation}\label{eq:pca}
    \{\mathbf{u}_1, \ldots, \mathbf{u}_K\},
\end{equation}
which capture the dominant, class-shared axes.  
PCA therefore yields a discriminative direction that is \emph{not tied to a 
particular class}, but instead reflects global sensitivity structure across 
all classes. This is shown in B of Figure~\ref{fig:ASF-S}. 


\textbf{Unit Projection and PGI Computation.}
Let $\mathbf{w}_i^{(\ell)}$ denote the flattened weight vector of unit $i$ in
layer $\ell$, where a “unit’’ refers to either a convolutional filter or a 
fully connected neuron.  
Let $\bar{\mathbf{w}}^{(f)}$ denote the mean weight vector of the output layer.
Each unit is scored via projection onto the top PCA components, weighted by the 
output-layer structure:
\begin{equation}
    s_{i,k}^{(\ell)} 
    = 
    \left\langle 
        \mathbf{w}_i^{(\ell)},\,
        \mathbf{u}_k \odot \bar{\mathbf{w}}^{(f)}
    \right\rangle.
\end{equation}
The PGI score for unit $i$ is
\begin{equation}
    \mathrm{PGI}_i^{(\ell)}
    =
    \mathcal{N}\!\left(
    \frac{1}{K} \sum_{k=1}^K 
        \left| s_{i,k}^{(\ell)} \right|
    \right),
\end{equation}
where $\mathcal{N}$ denotes min--max normalization across units in
layer~$\ell$. Low-PGI units contribute least to discriminative gradient flow and are pruned.
We then rebuild the pruned layers to address the dimensionality issues that arise during structured pruning of layers, preserving only the key components needed to maintain AFC so that the network can complete the classification task. 

Layer-specific thresholds $(\tau_1, \dots, \tau_L)$ are selected via grid search 
over PGI values, yielding multiple structured-pruned variants.

The proposed pruning method outlined above allows us to extract \textit{"Function"} for different architectures, based on a threshold $\tau_L$. The Lottery Ticket Hypothesis (LTH) \cite{frankle2019lottery} that we refer to in section \ref{sec:pre}, serves as a motivation to our method in the sense that we also approximate accurate smaller networks from larger ones. We will now show that by pruning the networks and retraining only the final layer, our method is able to recover accuracy up to a certain degree.

Let $w_i^{(f)}$ denote the weight vector of output unit $i$, $b_i^{(f)}$ the corresponding bias, and $f$ the output layer. In the architectures considered in this paper, the output layer is affine and hence it computes $f_i(z) = \big\langle w_i^{(f)}, z \big\rangle + b_i^{(f)},$ with prediction $ = \arg\max_i f_i(z)$.
Let $\tilde z(x)$ denote the feature produced by the pruned network at threshold configuration $\tau = (\tau_1,\dots,\tau_{L-1})$.

\begin{proposition}\label{prop:1}
For $x$ with label $y$, let $\gamma(x) = f_y(z(x)) - \max_{i\ne y}f_i(z(x))$ be the classification margin of the unpruned network, $B = \max_i\|w_i^{(f)}\|$, and $\varepsilon(\tau) = \|z(x)-\tilde z(x)\|$ the representation shift induced by pruning at threshold $\tau$. If $\varepsilon(\tau) < \gamma(x)/(2B)$, the unmodified output layer classifies $x$ correctly on $\tilde z(x)$.
\end{proposition}

As previously described, PGI-based pruning relies on computing a diagonal binary mask $M^{(l)}$ at layer $l$, zeroing units with $\mathrm{PGI}_i^{(l)} < \tau_l$. Proposition \ref{prop:1} gives a $\tau$-dependent sufficient condition for prediction preservation without retraining the pruned layers. In the results ahead we show how the representation shift varies with $\tau$ for sparse networks. 

Additionally, under the following Assumption \ref{assump:1}, PGI-based scoring in AFC-based pruning assigns greater importance to task-relevant units. I.e., following the principles of AFC, a unit's PGI score reflects its alignment with the network's task-discriminative behavior.
\begin{assumption}\label{assump:1}
Let $u_k$ be the $k$-th PCA component of the contrast matrix $\Delta G$ (Eq.\ref{eq:cont_matrix},\ref{eq:pca}) and $\bar
w^{(f)}$ the mean output-layer weight vector. A unit $i$ in layer $l$ is \emph{task-relevant} if
there exists $\kappa_i > 0$ such that
\begin{equation}
    \big\langle w_i^{(l)},\, u_k \odot \bar w^{(f)} - \bar w^{(f)} \big\rangle \;\ge\; \kappa_i.
\end{equation}
\end{assumption}
Assumption \ref{assump:1} states a sufficient condition under which the proposed method leads to a more task-aligned network. This can be written as following proposition:

\begin{proposition}
    Under Assumption \ref{assump:1}, PGI-based scoring assigns greater importance to units that are structurally aligned with the network's task-discriminative behavior. Formally, for every task-relevant unit $i$ in layer $l$, 
\begin{equation}
   s_{i,k}^{(l)} = \big\langle w_i^{(l)}, u_k\odot\bar w^{(f)}\big\rangle \ge \big\langle w_i^{(l)}, \bar w^{(f)}\big\rangle = p_{i,k}^{(l)}.
\end{equation}
\end{proposition}
The proofs of both propositions can be found in Appendix \ref{app:proofs}.

\section{Experiments \& Discussion}
\subsection{Evaluation and Datasets}
We trained three network architectures: \texttt{LeNet300\_100}, \texttt{Conv2Net}, and \texttt{Conv6Net} on MNIST and CIFAR-10 datasets, using standard preprocessing and augmentation pipelines. Models were optimized with Adam and early stopping based on validation loss, and final checkpoints corresponded to the lowest validation loss. Performance metrics, including accuracy, parameter count, and FLOPs can be found in section D of the Appendix. 

\begin{figure}[!h]
    \centering
    \begin{minipage}{0.4\textwidth}
    \centering
    \includegraphics[width=\linewidth]{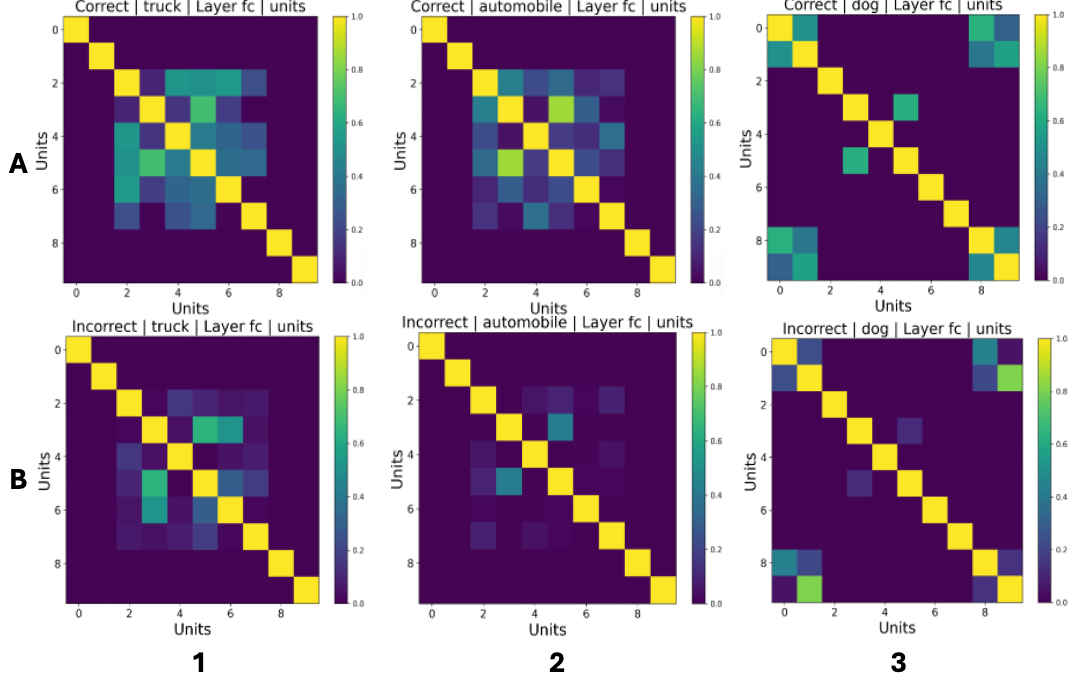}
    \caption{The FC maps extracted from the output layer of Conv2Net. 
    A1, A2 and A3 show PG1 for correctly classified ``truck'', ``automobile'', and ``dog'' 
    samples; B1, B2 and B3 show the corresponding incorrectly classified cases.}
    \label{fig:fc_maps}
    \end{minipage}%
\hfill
\begin{minipage}{0.4\textwidth}
\centering
    \includegraphics[width=\linewidth]{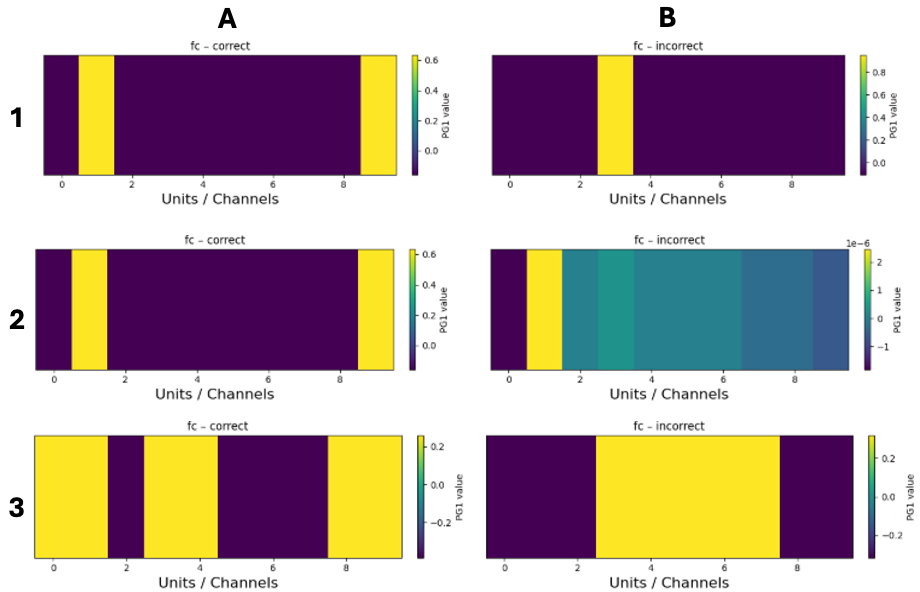}
    \caption{The principal gradients extracted from the output layer of Conv2Net. 
    A1, A2 and A3 show PG1 for correctly classified ``truck'', ``automobile'', and ``dog'' 
    samples; B1, B2 and B3 show the corresponding incorrectly classified cases.}
    \label{fig:pg_vectors}
\end{minipage}
\end{figure}


\subsection{Artificial Functional Connectivity \& Principle Gradient Importance Extraction}
The functional connectivity (FC) maps shown in Figure~\ref{fig:fc_maps} show a visual representation of the unit by unit similarity (or functional connectivity) of the output layer for Conv2Net. Upon visual inspection we found that classes that are semantically similar demonstrated similar looking maps. For example when comparing trucks and automobile (Figure~\ref{fig:fc_maps} (Column 1 and 2)) with an animal (column 3). In addition comparing the conditions of correct versus incorrect classification, the maps also demonstrate visible differences. PCA across classes is done when obtaining the PG1 contrast vector, due to the output layer having a significantly different topological organization for each class. A visualization of the PG1 vectors can be seen in Figure \ref{fig:pg_vectors}. In using the top components from PCA of the contrast vector, we are able to ensure that the artificial functional connectivity can be considered representative of all classes the model has been trained on as different classes yielded different PG1 vectors as shown in Figure \ref{fig:contrast_vector}. Further results can be found in section G of the Appendix where we give further discussion around the extracted FC maps and PG1 vectors. Our method prunes components using a single PGI score per component, similar to other inter-component/structured methods \cite{zniyed2025singular}. However this score is derived  based on a global phenomena of AFC that considers how structure and function in an ANN contribute to task performance. 

\begin{figure}[!h]
    \centering
    \includegraphics[width=0.8\linewidth]{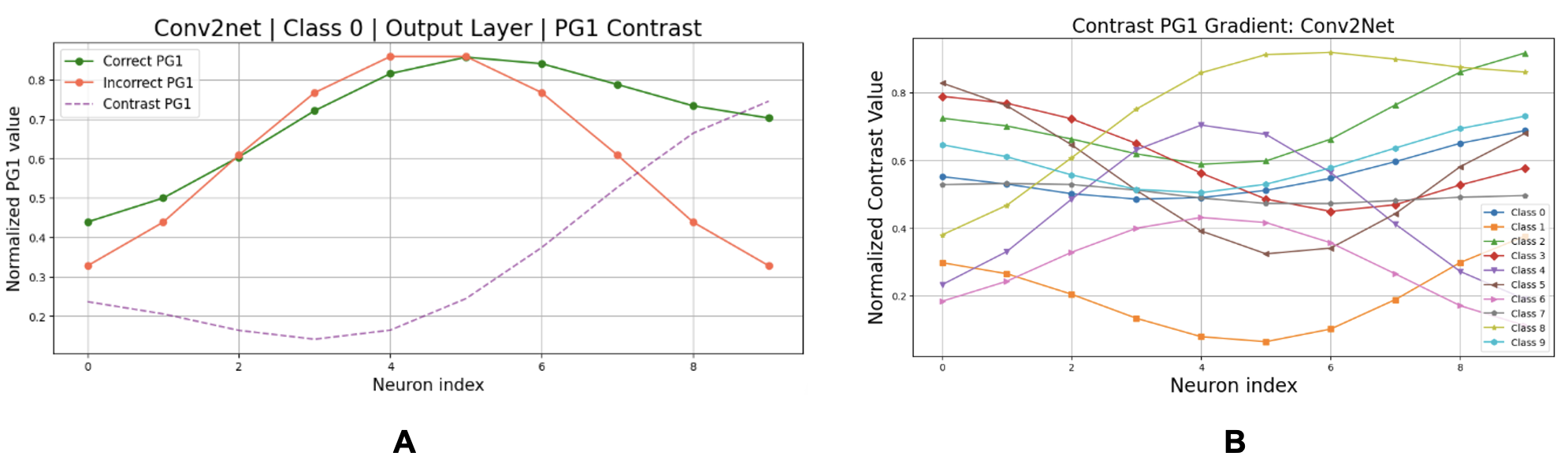}
    \caption{A shows the correct, incorrect gradients and the contrast vector extracted from the output layer of Conv2Net. B shows the contrast vectors extracted per class.}
    \label{fig:contrast_vector}
\end{figure}

\subsection{AFC For Interpretable Pruning Using ASF-S}
To evaluate AFC for pruning, we conducted post-training structured pruning experiments using ASF-S. We found that AFC enables simple, interpretable pruning strategies \cite{ling2024slimgpt,he2023structured} by using PG1 vectors to compute PGI scores and identify non-essential layer components. A threshold grid search on our convolutional network revealed multiple pruned variants with comparable sparsity but different accuracies (Figure~\ref{fig:conv2net_asf-s_variants}), indicating that pruning sensitivity varies by layer and that more aggressive pruning in some layers can be offset by retaining components in others. This suggests that an underlying AFC must be preserved through appropriate layer-wise threshold selection to maintain the structure-function relationship required for correct classification. PGI score results are shown in Figures~\ref{fig:lenet_pgi}, \ref{fig:conv2net_pgi}, and \ref{fig:conv6net_pgi} in the Appendix. Although PG1 primarily captures statistical co-fluctuations among neurons, pruning with respect to AFC provides task-relevant mechanistic insight by identifying the minimal subset of neurons required to preserve output-layer AFC and task performance. Units removable without altering PG contrast are less critical, while retained neurons form a task-relevant functional circuit, offering indirect mechanistic evidence of the components essential for task-specific processing.

\begin{figure}[!h]
    \centering
    \begin{minipage}{0.45\textwidth}
    \centering
    \includegraphics[width=\linewidth]{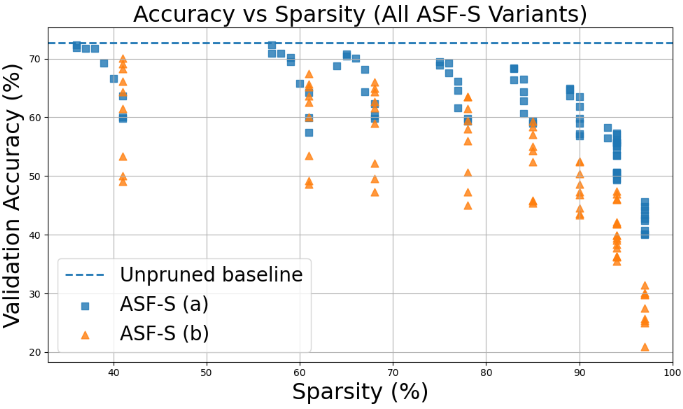}
    \caption{Accuracy-sparsity trade-off for all pruned variants obtained using a grid search for PGI thresholds per layer for ASF-S. ASF-S(a) corresponds to full model fine-tuning after pruning, (b) shows results for only fine-tuning the output layer. Each point corresponds to a pruned model instance. The dashed horizontal line denotes the unpruned baseline accuracy.}
    \label{fig:conv2net_asf-s_variants}
    \end{minipage}%
\hfill
\begin{minipage}{0.45\textwidth}
\centering
    \includegraphics[width=\linewidth]{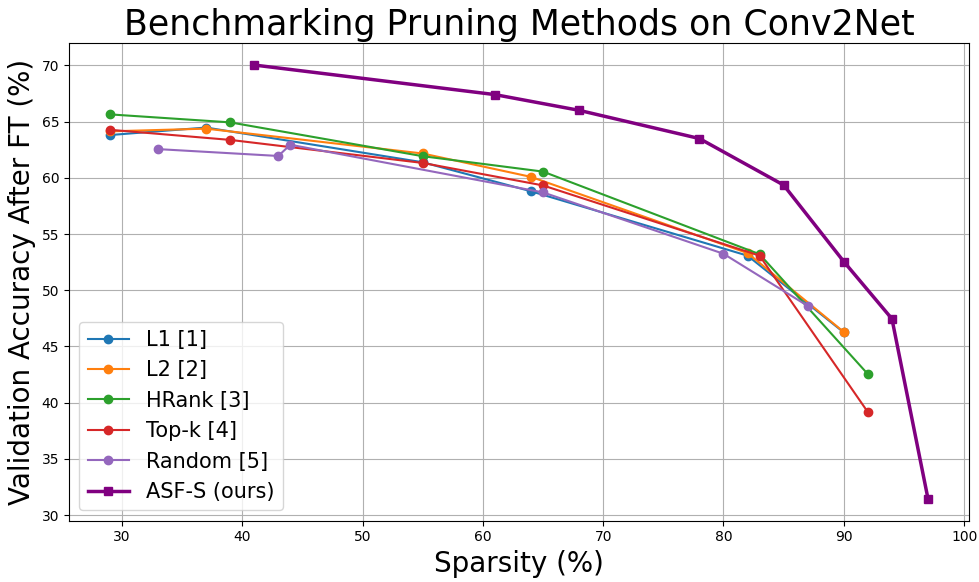}
    \caption{Our results for ASF-S following structured pruning and fine-tuning of only the unpruned output layer. In order to compare using PGI and PCA as a criterion for prune candidate selection we experiment with 3 other pruning criterion ([1],[2],[3]) \cite{han2015learning,han2015deep,lin2020hrank}. We also conducted experiments using a Top-k selection rule using an L1-style criterion \cite{he2023structured} ([4]). We also include random pruning for comparison.}
    \label{fig:fc_only_conv2net}
\end{minipage}
\end{figure}

\begin{table*}[!t]
\tiny
\centering
\caption{Comparison of pruning methods across models. Acc~$\downarrow$ and
\%~FLOPs~$\downarrow$ denote reduction. Lower Acc~$\downarrow$ is better,
while higher \%~FLOPs~$\downarrow$ is better. Datasets:
[1]~MNIST, [2]~CIFAR-10, [3]~CIFAR-100.}
\label{tab:all_methods_single}
\setlength{\tabcolsep}{2.5pt}
\renewcommand{\arraystretch}{1.05}
\resizebox{\textwidth}{!}{%
\begin{tabular}{l l c c c c c c}
\toprule
\textbf{Method} & \textbf{Metric}
    & \textbf{LeNet\_300\_100 [1]}
    & \textbf{Conv2Net [2]}
    & \textbf{Conv6Net [2]}
    & \textbf{Conv2Net [3]}
    & \textbf{Conv6Net [3]}
    & \textbf{VGG-16 [3]}\\
\midrule
\multirow{3}{*}{SVD~\cite{wang2019svd}}
    & Acc $\downarrow$         & ---   & 44.13 & 64.25 & 15.85 & 43.48 & 50.55 \\
    & \% Sparsity              & ---   & 70    & 70    & 70    & 70    & 70 \\
    & \% FLOPs $\downarrow$    & ---   & 42.1  & 27.2  & 87.81 & 89.76 & 90.79 \\

\midrule
\multirow{3}{*}{SLIMING~\cite{zniyed2025singular}}
    & Acc $\downarrow$         & ---   & 5.05  & 6.15 & 15.61 & 38.00 & 48.68 \\
    & \% Sparsity              & ---   & 70    & 70   & 70    & 70    & 70 \\
    & \% FLOPs $\downarrow$    & ---   & 90.50 & 17.2 & 87.81 & 89.76 & 90.79 \\
\midrule
\multirow{3}{*}{SNOWS~\cite{lucas2024preserving}}
    & Acc $\downarrow$         & ---   & 29.0 & 22.27  & 18.08 & 42.06 & 47.25 \\
    & \% Sparsity              & ---   & 70    & 70  & 70    & 70    & 70 \\
    & \% FLOPs $\downarrow$    & ---   & 63.45 & 66.60  & 87.81 & 89.76 & 90.79 \\
\midrule
\multirow{3}{*}{\textbf{ASF-S (ours)}}
    & Acc $\downarrow$         & \textbf{0}     & \textbf{7.74}  & \textbf{0}    & \textbf{3.55}  & \textbf{18.59} & \textbf{20.46} \\
    & \% Sparsity              & \textbf{70}    & \textbf{70}    & \textbf{70}   & \textbf{70}    & \textbf{70}    & \textbf{70} \\
    & \% FLOPs $\downarrow$    & \textbf{71.34} & \textbf{57.01} & \textbf{36.01} & \textbf{33.24} & \textbf{48.16} & \textbf{93.68}\\

\bottomrule
\end{tabular}%
}
\end{table*}
\subsection{PGI for Pruning Candidate Selection}
We present the accuracy and sparsity of the best CNN ASF-S model variants, against variants extracted using other pruning criterion in Figure~\ref{fig:fc_only_conv2net}. Post-training structured pruning methods can be applied without retraining the pruned layers, although fine-tuning is often used to recover performance. Despite being a post-training structured pruning method, ASF-S allows the network to recover accuracy through fine-tuning of only the output layer. This contrasts with conventional structured pruning approaches, where accuracy typically degrades requiring full retraining, or iterative training cycles or pruning done at initialization \cite{lee2018snip,frankle2019lottery,liu2019rethinking}. The benefit of using PGI scores as a criterion for candidate selection comes from the basis that these scores reflect a topographical organization (Function) that must be preserved within the output layer. Our results demonstrate the benefit of this criterion as opposed to those included in Figure~\ref{fig:fc_only_conv2net}, as we are able to return much closer to baseline accuracy (within 10\%) at sparsity levels up to 80\%., all without the need for iterative pruning cycles. All benchmarks criterion followed the same fine-tuning strategy, with pruned layers frozen. We present benchmark results against other model compression techniques also in Table~\ref{tab:all_methods_single}, where we show reduction in model accuracy and FLOPs at 70\% sparsity levels as an example. Here a smaller reduction in accuracy is best as we demonstrate the effectiveness of preserving model performance despite structured pruning without fine-tuning of the pruned layers. Our method also yields comparable reductions in FLOPs demonstrating its effectiveness of reducing computational cost as measured by FLOPs, more results regarding this can also be found in section E.2 of the Appendix.  

\subsection{ASF-S On A Transformer Architecture}
To test whether preserving functional organisation transfers beyond vision models, we extend evaluation to a Transformer classifier (\textsc{BERT-base}) on \textsc{AG News}. We prune FFN intermediate neurons in selected encoder blocks (\{0,5,11\}) and compare PG-guided pruning (ASF-S) against weight-driven structured baselines (L1/L2) and random pruning at matched sparsity and identical \emph{head-only} fine-tuning budgets. Across moderate-to-high sparsity (60--70\%), PG-guided pruning yields stronger accuracy recovery than random and remains competitive with L1/L2, while preserving downstream performance without requiring full-model retraining. We additionally report calibration (ECE), highlighting that aggressive FFN pruning can shift confidence even when accuracy is largely preserved; this motivates reporting \emph{accuracy--calibration trade-offs} rather than accuracy alone for deployment-oriented pruning. Further results can be found in section G.4 of the Appendix.

\section{Conclusion \& Limitations}
We introduce Artificial Structure Function Search (ASF-S), a neuroscience-inspired structured pruning framework that identifies accurate sparse networks from fully trained ANNs. Building on the Lottery Ticket Hypothesis \cite{frankle2019lottery}, ASF-S is guided by artificial functional connectivity (AFC), which captures the proposed relationship between network structure (prunable layer components) and function (the topographical organization of the output layer enabling classification). We propose that preserving AFC during pruning can guide layer removal without requiring post-prune fine-tuning. ASF-S pruned variants recover accuracy to within 10\% of baseline without fine-tuning pruned layers, suggesting the retained structure preserves the functional relationships driving task performance. This offers a more interpretable and energy-efficient approach to model compression, with future work exploring iterative pruning, larger-scale validation, and improved PGI thresholding.

Despite promising results, our method has some limitations. First, computational constraints prevented evaluation on large-scale datasets such as ImageNet, leaving scalability an open question. Second, the compatibility of our approach with residual architectures (e.g., ResNet) remains untested, as skip connections introduce structural dependencies that may affect pruning efficacy. Third, unlike many structured pruning techniques that employ iterative pruning regimes, our method performs a single pruning step on a pre-trained model thus the effect of an iterative variant is yet to be explored. Finally, storing per-class activations incurs significant disk overhead, which adds to the challenge of scaling to more complex datasets.
\bibliographystyle{abbrvnat}
\bibliography{ref}

@article{cheng2023survey,
  title={A survey on deep neural network pruning-taxonomy, comparison},
  author={Cheng, Hongrong and Zhang, Miao and Shi, Javen Qinfeng},
  journal={Analysis, and Recommendations},
  year={2023}
}

@inproceedings{devlin2019bert,
  title={Bert: Pre-training of deep bidirectional transformers for language understanding},
  author={Devlin, Jacob and Chang, Ming-Wei and Lee, Kenton and Toutanova, Kristina},
  booktitle={Proceedings of the 2019 conference of the North American chapter of the association for computational linguistics: human language technologies, volume 1 (long and short papers)},
  pages={4171--4186},
  year={2019}
}

@article{simonyan2014very,
  title={Very deep convolutional networks for large-scale image recognition},
  author={Simonyan, Karen and Zisserman, Andrew},
  journal={arXiv preprint arXiv:1409.1556},
  year={2014}
}

@article{han2015deep,
  title={Deep compression: Compressing deep neural networks with pruning, trained quantization and huffman coding},
  author={Han, Song and Mao, Huizi and Dally, William J},
  journal={arXiv preprint arXiv:1510.00149},
  year={2015}
}

@article{han2015learning,
  title={Learning both weights and connections for efficient neural network},
  author={Han, Song and Pool, Jeff and Tran, John and Dally, William},
  journal={Advances in neural information processing systems},
  volume={28},
  year={2015}
}

@inproceedings{luo2017thinet,
  title={Thinet: A filter level pruning method for deep neural network compression},
  author={Luo, Jian-Hao and Wu, Jianxin and Lin, Weiyao},
  booktitle={Proceedings of the IEEE international conference on computer vision},
  pages={5058--5066},
  year={2017}
}

@article{sehwag2020hydra,
  title={Hydra: Pruning adversarially robust neural networks},
  author={Sehwag, Vikash and Wang, Shiqi and Mittal, Prateek and Jana, Suman},
  journal={Advances in Neural Information Processing Systems},
  volume={33},
  pages={19655--19666},
  year={2020}
}

@article{ashkboos2024slicegpt,
  title={Slicegpt: Compress large language models by deleting rows and columns},
  author={Ashkboos, Saleh and Croci, Maximilian L and Nascimento, Marcelo Gennari do and Hoefler, Torsten and Hensman, James},
  journal={arXiv preprint arXiv:2401.15024},
  year={2024}
}

@article{ma2023llm,
  title={Llm-pruner: On the structural pruning of large language models},
  author={Ma, Xinyin and Fang, Gongfan and Wang, Xinchao},
  journal={Advances in neural information processing systems},
  volume={36},
  pages={21702--21720},
  year={2023}
}

@inproceedings{li2021npas,
  title={Npas: A compiler-aware framework of unified network pruning and architecture search for beyond real-time mobile acceleration},
  author={Li, Zhengang and Yuan, Geng and Niu, Wei and Zhao, Pu and Li, Yanyu and Cai, Yuxuan and Shen, Xuan and Zhan, Zheng and Kong, Zhenglun and Jin, Qing and others},
  booktitle={Proceedings of the IEEE/CVF Conference on Computer Vision and Pattern Recognition},
  pages={14255--14266},
  year={2021}
}

@article{he2023structured,
  title={Structured pruning for deep convolutional neural networks: A survey},
  author={He, Yang and Xiao, Lingao},
  journal={IEEE transactions on pattern analysis and machine intelligence},
  volume={46},
  number={5},
  pages={2900--2919},
  year={2023},
  publisher={IEEE}
}

@article{morcos2019one,
  title={One ticket to win them all: generalizing lottery ticket initializations across datasets and optimizers},
  author={Morcos, Ari and Yu, Haonan and Paganini, Michela and Tian, Yuandong},
  journal={Advances in neural information processing systems},
  volume={32},
  year={2019}
}

@article{lee2018snip,
  title={Snip: Single-shot network pruning based on connection sensitivity},
  author={Lee, Namhoon and Ajanthan, Thalaiyasingam and Torr, Philip HS},
  journal={arXiv preprint arXiv:1810.02340},
  year={2018}
}

@article{ronde2024default,
  title={Default mode network dynamics: An integrated neurocircuitry perspective on social dysfunction in human brain disorders},
  author={Ronde, Mirthe and van der Zee, Eddy A and Kas, Martien JH},
  journal={Neuroscience \& Biobehavioral Reviews},
  volume={164},
  pages={105839},
  year={2024},
  publisher={Elsevier}
}

@article{smallwood2021default,
  title={The default mode network in cognition: a topographical perspective},
  author={Smallwood, Jonathan and Bernhardt, Boris C and Leech, Robert and Bzdok, Danilo and Jefferies, Elizabeth and Margulies, Daniel S},
  journal={Nature reviews neuroscience},
  volume={22},
  number={8},
  pages={503--513},
  year={2021},
  publisher={Nature Publishing Group UK London}
}

@article{margulies2016situating,
  title={Situating the default-mode network along a principal gradient of macroscale cortical organization},
  author={Margulies, Daniel S and Ghosh, Satrajit S and Goulas, Alexandros and Falkiewicz, Marcel and Huntenburg, Julia M and Langs, Georg and Bezgin, Gleb and Eickhoff, Simon B and Castellanos, F Xavier and Petrides, Michael and others},
  journal={Proceedings of the National Academy of Sciences},
  volume={113},
  number={44},
  pages={12574--12579},
  year={2016},
  publisher={National Academy of Sciences}
}

@inproceedings{evci2020rigging,
  title={Rigging the lottery: Making all tickets winners},
  author={Evci, Utku and Gale, Trevor and Menick, Jacob and Castro, Pablo Samuel and Elsen, Erich},
  booktitle={International conference on machine learning},
  pages={2943--2952},
  year={2020},
  organization={PMLR}
}

@inproceedings{frantar2023sparsegpt,
  title={Sparsegpt: Massive language models can be accurately pruned in one-shot},
  author={Frantar, Elias and Alistarh, Dan},
  booktitle={International conference on machine learning},
  pages={10323--10337},
  year={2023},
  organization={PMLR}
}

@article{you2019gate,
  title={Gate decorator: Global filter pruning method for accelerating deep convolutional neural networks},
  author={You, Zhonghui and Yan, Kun and Ye, Jinmian and Ma, Meng and Wang, Ping},
  journal={Advances in neural information processing systems},
  volume={32},
  year={2019}
}

@article{shao2023omniquant,
  title={Omniquant: Omnidirectionally calibrated quantization for large language models},
  author={Shao, Wenqi and Chen, Mengzhao and Zhang, Zhaoyang and Xu, Peng and Zhao, Lirui and Li, Zhiqian and Zhang, Kaipeng and Gao, Peng and Qiao, Yu and Luo, Ping},
  journal={arXiv preprint arXiv:2308.13137},
  year={2023}
}

@article{gu2023minillm,
  title={Minillm: Knowledge distillation of large language models},
  author={Gu, Yuxian and Dong, Li and Wei, Furu and Huang, Minlie},
  journal={arXiv preprint arXiv:2306.08543},
  year={2023}
}

@article{zhao2023task,
  title={Task fMRI paradigms may capture more behaviorally relevant information than resting-state functional connectivity},
  author={Zhao, Weiqi and Makowski, Carolina and Hagler, Donald J and Garavan, Hugh P and Thompson, Wesley K and Greene, Deanna J and Jernigan, Terry L and Dale, Anders M},
  journal={NeuroImage},
  volume={270},
  pages={119946},
  year={2023},
  publisher={Elsevier}
}

@article{frankle2019lottery,
  title={The lottery ticket hypothesis: Finding sparse, trainable neural networks},
  author={Frankle, Jonathan and Carbin, Michael},
  booktitle={International Conference on Learning Representations (ICLR)},
  year={2019},
  url={https://arxiv.org/abs/1803.03635}
}

@article{paquola2025architecture,
  title={The architecture of the human default mode network explored through cytoarchitecture, wiring and signal flow},
  author={Paquola, Casey and Garber, Margaret and Fr{\"a}ssle, Stefan and Royer, Jessica and Zhou, Yigu and Tavakol, Shahin and Rodriguez-Cruces, Raul and Cabalo, Donna Gift and Valk, Sofie and Eickhoff, Simon B and others},
  journal={Nature neuroscience},
  volume={28},
  number={3},
  pages={654--664},
  year={2025},
  publisher={Nature Publishing Group US New York}
}

@article{lecun1989optimal,
  title={Optimal brain damage},
  author={LeCun, Yann and Denker, John and Solla, Sara},
  journal={Advances in neural information processing systems},
  volume={2},
  year={1989}
}

@article{liu2019rethinking,
  title={Rethinking the Value of Network Pruning},
  author={Liu, Zhuang and Sun, Mingjie and Zhou, Tinghui and Huang, Gao and Darrell, Trevor},
  journal={arXiv preprint arXiv:1810.05270},
  year={2018},
  url={https://arxiv.org/abs/1810.05270}
}

@article{han2016eie,
  title={EIE: Efficient inference engine on compressed deep neural network},
  author={Han, Song and Liu, Xingyu and Mao, Huizi and Pu, Jing and Pedram, Ardavan and Horowitz, Mark A and Dally, William J},
  journal={ACM SIGARCH Computer Architecture News},
  volume={44},
  number={3},
  pages={243--254},
  year={2016},
  publisher={ACM New York, NY, USA}
}

@inproceedings{wang2020grasp,
  title={Picking Winning Tickets Before Training by Preserving Gradient Flow},
  author={Wang, Chaoqi and Zhang, Guodong and Grosse, Roger},
  booktitle={International Conference on Learning Representations (ICLR)},
  year={2020},
  url={https://openreview.net/forum?id=Hyg6WgHKwS}
}

@article{guo2025slimllm,
  title={SlimLLM: Accurate Structured Pruning for Large Language Models},
  author={Guo, Jialong and Chen, Xinghao and Tang, Yehui and Wang, Yunhe},
  journal={arXiv preprint arXiv:2505.22689},
  year={2025}
}

@article{fox2007spontaneous,
  title={Spontaneous fluctuations in brain activity observed with functional magnetic resonance imaging},
  author={Fox, Michael D and Raichle, Marcus E},
  journal={Nature reviews neuroscience},
  volume={8},
  number={9},
  pages={700--711},
  year={2007},
  publisher={Nature Publishing Group UK London}
}

@article{cohen2008defining,
  title={Defining functional areas in individual human brains using resting functional connectivity MRI},
  author={Cohen, Alexander L and Fair, Damien A and Dosenbach, Nico UF and Miezin, Francis M and Dierker, Donna and Van Essen, David C and Schlaggar, Bradley L and Petersen, Steven E},
  journal={Neuroimage},
  volume={41},
  number={1},
  pages={45--57},
  year={2008},
  publisher={Elsevier}
}

@article{van2010exploring,
  title={Exploring the brain network: a review on resting-state fMRI functional connectivity},
  author={Van Den Heuvel, Martijn P and Pol, Hilleke E Hulshoff},
  journal={European neuropsychopharmacology},
  volume={20},
  number={8},
  pages={519--534},
  year={2010},
  publisher={Elsevier}
}

@article{lecun2002gradient,
  title={Gradient-based learning applied to document recognition},
  author={LeCun, Yann and Bottou, L{\'e}on and Bengio, Yoshua and Haffner, Patrick},
  journal={Proceedings of the IEEE},
  volume={86},
  number={11},
  pages={2278--2324},
  year={2002},
  publisher={Ieee}
}

@article{ling2024slimgpt,
  title={Slimgpt: Layer-wise structured pruning for large language models},
  author={Ling, Gui and Wang, Ziyang and Liu, Qingwen},
  journal={Advances in Neural Information Processing Systems},
  volume={37},
  pages={107112--107137},
  year={2024}
}

@inproceedings{zeiler2014visualizing,
  title={Visualizing and understanding convolutional networks},
  author={Zeiler, Matthew D and Fergus, Rob},
  booktitle={European conference on computer vision},
  pages={818--833},
  year={2014},
  organization={Springer}
}

@inproceedings{he2017channel,
  title={Channel pruning for accelerating very deep neural networks},
  author={He, Yihui and Zhang, Xiangyu and Sun, Jian},
  booktitle={Proceedings of the IEEE international conference on computer vision},
  pages={1389--1397},
  year={2017}
}

@inproceedings{tian2024visual,
  title={Visual AutoRegressive Modeling: Scalable Image Generation via Next-Scale Prediction},
  author={Tian, Keyu and Jiang, Yi and Yuan, Zehuan and Peng, Bingyue and Wang, Liwei},
  booktitle={Advances in Neural Information Processing Systems (NeurIPS)},
  year={2024},
  url={https://proceedings.neurips.cc/paper_files/paper/2024/file/9a24e284b187f662681440ba15c416fb-Paper-Conference.pdf}
}

@article{Greenacre2022_PCA,
  author = {Greenacre, Michael and Groenen, Patrick J. F. and Hastie, Trevor and D'Enza, Alfonso Iodice and Markos, Angelos and Tuzhilina, Elena},
  title = {Principal component analysis},
  journal = {Nature Reviews Methods Primers},
  volume = {2},
  number = {1},
  pages = {100},
  year = {2022},
  doi = {10.1038/s43586-022-00184-w}
}

@article{Coifman2006_DiffusionMaps,
  author = {Coifman, Ronald R. and Lafon, Stéphane},
  title = {Diffusion maps},
  journal = {Applied and Computational Harmonic Analysis},
  volume = {21},
  number = {1},
  pages = {5--30},
  year = {2006},
  doi = {10.1016/j.acha.2006.04.006}
}

@article{Yeo2011_7Networks,
  author = {Yeo, B.T.T. and Krienen, F.M. and Sepulcre, J. and Sabuncu, M.R. and Lashkari, D. and Hollinshead, M. and Roffman, J.L. and Smoller, J.W. and Zollei, L. and Polimeni, J.R. and Fischl, B. and Liu, H. and Buckner, R.L.},
  title = {The organization of the human cerebral cortex estimated by intrinsic functional connectivity},
  journal = {Journal of Neurophysiology},
  volume = {106},
  number = {3},
  pages = {1125--1165},
  year = {2011},
  doi = {10.1152/jn.00338.2011}
}

@article{oliver2019quantifying,
  title={Quantifying the variability in resting-state networks},
  author={Oliver, Isaura and Hlinka, Jaroslav and Kopal, Jakub and Davidsen, J{\"o}rn},
  journal={Entropy},
  volume={21},
  number={9},
  pages={882},
  year={2019},
  publisher={MDPI}
}

@article{kriegeskorte2008representational,
  title={Representational similarity analysis – connecting the branches of systems neuroscience},
  author={Kriegeskorte, Nikolaus and Mur, Marieke and Bandettini, Peter},
  journal={Frontiers in Systems Neuroscience},
  volume={2},
  pages={4},
  year={2008},
  doi={10.3389/neuro.06.004.2008}
}

@inproceedings{kornblith2019similarity,
  title={Similarity of Neural Network Representations Revisited},
  author={Kornblith, Simon and Norouzi, Mohammad and Lee, Honglak and Hinton, Geoffrey},
  booktitle={Proceedings of the 36th International Conference on Machine Learning (ICML)},
  year={2019},
  pages={3519--3529},
  url={http://proceedings.mlr.press/v97/kornblith19a.html}
}

@inproceedings{zhang2021hardware,
  title={Hardware-software codesign of weight reshaping and systolic array multiplexing for efficient CNNs},
  author={Zhang, Jingyao and Gu, Huaxi and Zhang, Grace Li and Li, Bing and Schlichtmann, Ulf},
  booktitle={2021 Design, Automation \& Test in Europe Conference \& Exhibition (DATE)},
  pages={667--672},
  year={2021},
  organization={IEEE}
}

@article{pachauri2014permutation,
  title={Permutation diffusion maps (pdm) with application to the image association problem in computer vision},
  author={Pachauri, Deepti and Kondor, Risi and Sargur, Gautam and Singh, Vikas},
  journal={Advances in Neural Information Processing Systems},
  volume={27},
  year={2014}
}

@article{zhang2018diffusion,
  title={Diffusion maps for textual network embedding},
  author={Zhang, Xinyuan and Li, Yitong and Shen, Dinghan and Carin, Lawrence},
  journal={Advances in Neural Information Processing Systems},
  volume={31},
  year={2018}
}

@article{valeriani2023geometry,
  title={The geometry of hidden representations of large transformer models},
  author={Valeriani, Lucrezia and Doimo, Diego and Cuturello, Francesca and Laio, Alessandro and Ansuini, Alessio and Cazzaniga, Alberto},
  journal={Advances in Neural Information Processing Systems},
  volume={36},
  pages={51234--51252},
  year={2023}
}

@article{kraidia2024defense,
  title={Defense against adversarial attacks: robust and efficient compressed optimized neural networks},
  author={Kraidia, Insaf and Ghenai, Afifa and Belhaouari, Samir Brahim},
  journal={Scientific Reports},
  volume={14},
  number={1},
  pages={6420},
  year={2024},
  publisher={Nature Publishing Group UK London}
}

@article{demir2025overview,
  title={Overview of Memory-Efficient Architectures for Deep Learning in Real-Time Systems},
  author={Demir, Bilgin and Domazet, Ervin and Mechkaroska, Daniela},
  journal={Engineering Proceedings},
  volume={104},
  number={1},
  pages={77},
  year={2025},
  publisher={MDPI}
}

@inproceedings{malach2020proving,
  title={Proving the lottery ticket hypothesis: Pruning is all you need},
  author={Malach, Eran and Yehudai, Gilad and Shalev-Schwartz, Shai and Shamir, Ohad},
  booktitle={International Conference on Machine Learning},
  pages={6682--6691},
  year={2020},
  organization={PMLR}
}

@article{pensia2020optimal,
  title={Optimal lottery tickets via subset sum: Logarithmic over-parameterization is sufficient},
  author={Pensia, Ankit and Rajput, Shashank and Nagle, Alliot and Vishwakarma, Harit and Papailiopoulos, Dimitris},
  journal={Advances in neural information processing systems},
  volume={33},
  pages={2599--2610},
  year={2020}
}

@article{chen2020lottery,
  title={The lottery ticket hypothesis for pre-trained bert networks},
  author={Chen, Tianlong and Frankle, Jonathan and Chang, Shiyu and Liu, Sijia and Zhang, Yang and Wang, Zhangyang and Carbin, Michael},
  journal={Advances in neural information processing systems},
  volume={33},
  pages={15834--15846},
  year={2020}
}

@article{yao2023probabilistic,
  title={Probabilistic modeling: Proving the lottery ticket hypothesis in spiking neural network},
  author={Yao, Man and Chou, Yuhong and Zhao, Guangshe and Zheng, Xiawu and Tian, Yonghong and Xu, Bo and Li, Guoqi},
  journal={arXiv preprint arXiv:2305.12148},
  year={2023}
}

@article{fotiadis2024structure,
  title={Structure--function coupling in macroscale human brain networks},
  author={Fotiadis, Panagiotis and Parkes, Linden and Davis, Kathryn A and Satterthwaite, Theodore D and Shinohara, Russell T and Bassett, Dani S},
  journal={Nature Reviews Neuroscience},
  volume={25},
  number={10},
  pages={688--704},
  year={2024},
  publisher={Nature Publishing Group UK London}
}

@article{uddin2008residual,
  title={Residual functional connectivity in the split-brain revealed with resting-state functional MRI},
  author={Uddin, Lucina Q and Mooshagian, Eric and Zaidel, Eran and Scheres, Anouk and Margulies, Daniel S and Kelly, AM Clare and Shehzad, Zarrar and Adelstein, Jonathan S and Castellanos, F Xavier and Biswal, Bharat B and others},
  journal={Neuroreport},
  volume={19},
  number={7},
  pages={703--709},
  year={2008},
  publisher={LWW}
}

@article{uddin2013complex,
  title={Complex relationships between structural and functional brain connectivity},
  author={Uddin, Lucina Q},
  journal={Trends in cognitive sciences},
  volume={17},
  number={12},
  pages={600--602},
  year={2013},
  publisher={Elsevier}
}

@article{whitfield2012conn,
  title={Conn: a functional connectivity toolbox for correlated and anticorrelated brain networks},
  author={Whitfield-Gabrieli, Susan and Nieto-Castanon, Alfonso},
  journal={Brain connectivity},
  volume={2},
  number={3},
  pages={125--141},
  year={2012},
  publisher={Mary Ann Liebert, Inc. 140 Huguenot Street, 3rd Floor New Rochelle, NY 10801 USA}
}

@article{esteban2019fmriprep,
  title={fMRIPrep: a robust preprocessing pipeline for functional MRI},
  author={Esteban, Oscar and Markiewicz, Christopher J and Blair, Ross W and Moodie, Craig A and Isik, A Ilkay and Erramuzpe, Asier and Kent, James D and Goncalves, Mathias and DuPre, Elizabeth and Snyder, Madeleine and others},
  journal={Nature methods},
  volume={16},
  number={1},
  pages={111--116},
  year={2019},
  publisher={Nature Publishing Group US New York}
}

@article{sinha2023intracranial,
  title={Intracranial EEG structure-function coupling and seizure outcomes after epilepsy surgery},
  author={Sinha, Nishant and Duncan, John S and Diehl, Beate and Chowdhury, Fahmida A and De Tisi, Jane and Miserocchi, Anna and McEvoy, Andrew William and Davis, Kathryn A and Vos, Sjoerd B and Winston, Gavin P and others},
  journal={Neurology},
  volume={101},
  number={13},
  pages={e1293--e1306},
  year={2023},
  publisher={Lippincott Williams \& Wilkins Hagerstown, MD}
}

@article{honey2010can,
  title={Can structure predict function in the human brain?},
  author={Honey, Christopher J and Thivierge, Jean-Philippe and Sporns, Olaf},
  journal={Neuroimage},
  volume={52},
  number={3},
  pages={766--776},
  year={2010},
  publisher={Elsevier}
}

@article{mesulam1998sensation,
  title={From sensation to cognition.},
  author={Mesulam, M-Marsel},
  journal={Brain: a journal of neurology},
  volume={121},
  number={6},
  pages={1013--1052},
  year={1998}
}

@article{zhang2023functional,
  title={Functional network: A novel framework for interpretability of deep neural networks},
  author={Zhang, Ben and Dong, Zhetong and Zhang, Junsong and Lin, Hongwei},
  journal={Neurocomputing},
  volume={519},
  pages={94--103},
  year={2023},
  publisher={Elsevier}
}

@inproceedings{songdechakraiwut2025functional,
  title={Functional connectomes of neural networks},
  author={Songdechakraiwut, Tananun and Wu, Yutong},
  booktitle={Proceedings of the AAAI Conference on Artificial Intelligence},
  volume={39},
  number={19},
  pages={20558--20566},
  year={2025}
}

@article{dapello2020simulating,
  title={Simulating a primary visual cortex at the front of CNNs improves robustness to image perturbations},
  author={Dapello, Joel and Marques, Tiago and Schrimpf, Martin and Geiger, Franziska and Cox, David and DiCarlo, James J},
  journal={Advances in Neural Information Processing Systems},
  volume={33},
  pages={13073--13087},
  year={2020}
}

@article{janusz2025one,
  title={One Shot vs. Iterative: Rethinking Pruning Strategies for Model Compression},
  author={Janusz, Mikolaj and Wojnar, Tomasz and Li, Yawei and Benini, Luca and Adamczewski, Kamil},
  journal={arXiv preprint arXiv:2508.13836},
  year={2025}
}

@inproceedings{blalock2020state,
  title={What is the State of Neural Network Pruning?},
  author={Blalock, Davis and Gonzalez Ortiz, Jose Javier and Frankle, Jonathan and Guttag, John},
  booktitle={Proceedings of Machine Learning and Systems (MLSys)},
  year={2020}
}

@article{zhang2018intrinsic,
  title={Intrinsic neural linkage between primary visual area and default mode network in human brain: evidence from visual mental imagery},
  author={Zhang, Zheng and Zhang, Delong and Wang, Zengjian and Li, Junchao and Lin, Yuting and Chang, Song and Huang, Ruiwang and Liu, Ming},
  journal={Neuroscience},
  volume={379},
  pages={13--21},
  year={2018},
  publisher={Elsevier}
}

@article{guo2024interpretable,
  title={Interpretable task-inspired adaptive filter pruning for neural networks under multiple constraints},
  author={Guo, Yang and Gao, Wei and Li, Ge},
  journal={International Journal of Computer Vision},
  volume={132},
  number={6},
  pages={2060--2076},
  year={2024},
  publisher={Springer}
}

@inproceedings{lin2020hrank,
  title={Hrank: Filter pruning using high-rank feature map},
  author={Lin, Mingbao and Ji, Rongrong and Wang, Yan and Zhang, Yichen and Zhang, Baochang and Tian, Yonghong and Shao, Ling},
  booktitle={Proceedings of the IEEE/CVF conference on computer vision and pattern recognition},
  pages={1529--1538},
  year={2020}
}

@article{zniyed2025singular,
  title={Singular values-driven automated filter pruning},
  author={Zniyed, Yassine and Nguyen, Thanh Phuong and others},
  journal={Neural Networks},
  pages={107857},
  year={2025},
  publisher={Elsevier}
}

@article{fukushima1983neocognitron,
  title={Neocognitron: A neural network model for a mechanism of visual pattern recognition},
  author={Fukushima, Kunihiko and Miyake, Sei and Ito, Takayuki},
  journal={IEEE transactions on systems, man, and cybernetics},
  number={5},
  pages={826--834},
  year={1983},
  publisher={IEEE}
}

@article{yamins2014performance,
  title={Performance-optimized hierarchical models predict neural responses in higher visual cortex},
  author={Yamins, Daniel LK and Hong, Ha and Cadieu, Charles F and Solomon, Ethan A and Seibert, Darren and DiCarlo, James J},
  journal={Proceedings of the national academy of sciences},
  volume={111},
  number={23},
  pages={8619--8624},
  year={2014},
  publisher={National Academy of Sciences}
}

@inproceedings{bau2017network,
  title={Network dissection: Quantifying interpretability of deep visual representations},
  author={Bau, David and Zhou, Bolei and Khosla, Aditya and Oliva, Aude and Torralba, Antonio},
  booktitle={Proceedings of the IEEE conference on computer vision and pattern recognition},
  pages={6541--6549},
  year={2017}
}

@article{suarez2020linking,
  title={Linking structure and function in macroscale brain networks},
  author={Su{\'a}rez, Laura E and Markello, Ross D and Betzel, Richard F and Misic, Bratislav},
  journal={Trends in cognitive sciences},
  volume={24},
  number={4},
  pages={302--315},
  year={2020},
  publisher={Elsevier}
}

@article{mivsic2016network,
  title={Network-level structure-function relationships in human neocortex},
  author={Mi{\v{s}}i{\'c}, Bratislav and Betzel, Richard F and De Reus, Marcel A and Van Den Heuvel, Martijn P and Berman, Marc G and McIntosh, Anthony R and Sporns, Olaf},
  journal={Cerebral Cortex},
  volume={26},
  number={7},
  pages={3285--3296},
  year={2016},
  publisher={Oxford University Press}
}

@article{jiao2025touching,
  title={“Touching” the brain: braille reading mitigates the SC--FC decoupling of brain networks in congenital blindness},
  author={Jiao, Saiyi and Wang, Ke and Zeng, Jiahong and Cui, Zhenjiang and Luo, Yudan and Han, Zaizhu},
  journal={Brain Structure and Function},
  volume={230},
  number={6},
  pages={114},
  year={2025},
  publisher={Springer}
}

@inproceedings{laenen2023one,
  title={One-shot neural network pruning via spectral graph sparsification},
  author={Laenen, Steinar},
  booktitle={Topological, Algebraic and Geometric Learning Workshops 2023},
  pages={60--71},
  year={2023},
  organization={PMLR}
}

@article{alon1986eigenvalues,
  title     = {Eigenvalues and Expanders},
  author    = {Alon, Noga},
  journal   = {Combinatorica},
  volume    = {6},
  number    = {2},
  pages     = {83--96},
  year      = {1986}
}

@article{liu2025structure,
  title={Structure-Aware Automatic Channel Pruning by Searching with Graph Embedding},
  author={Liu, Zifan and Cao, Yuan and Yu, Yanwei and Qi, Heng and Gui, Jie},
  journal={arXiv preprint arXiv:2506.11469},
  year={2025}
}

@article{lucas2024preserving,
  title={Preserving deep representations in one-shot pruning: A hessian-free second-order optimization framework},
  author={Lucas, Ryan and Mazumder, Rahul},
  journal={arXiv preprint arXiv:2411.18376},
  year={2024}
}

@inproceedings{wang2019svd,
  title={SVD-based channel pruning for convolutional neural network in acoustic scene classification model},
  author={Wang, Jun and Li, Shengchen and Wang, Wenwu},
  booktitle={2019 IEEE International Conference on Multimedia \& Expo Workshops (ICMEW)},
  pages={390--395},
  year={2019},
  organization={IEEE}
}

\newpage
\makeatletter
\@addtoreset{theorem}{section}
\@addtoreset{proposition}{section}
\@addtoreset{assumption}{section}
\makeatother
\appendix
\counterwithin{subsection}{section}

\section{Additional Theoretical Background}
\textbf{Topological Organization and Structure Function Relationships In Neuroscience.}
In Neuroscience, topological organization can refer to structural or functional connectivity between network components, this concept is explored for its role in determining the processing behind cognition and output. Structural connectivity is usually determined by calculating the density of fibrous connections between two brain regions, the greater the density the greater structural connectivity two such regions will have. Functional connectivity on the other hand, is derived from observing the statistical similarity between neuronal activity timeseries, that can be extracted using various neuroimaging methods including via fluctuations in the BOLD signal from functional magnetic resonance imaging (fMRI) \cite{sinha2023intracranial,whitfield2012conn,esteban2019fmriprep}. The resultant relationship between structural connectivity and functional connectivity is then measured using structure-function coupling (SFC) techniques \cite{honey2010can}, to understand how structure and function give rise to tasks like object recognition \cite{jiao2025touching,suarez2020linking,yamins2014performance,mivsic2016network}.

\textbf{Intermediate Processing In the Complex Organisms}
These topological organization principles in the brain are believed to facilitate the intermediate processing that enables organisms to take in stimuli from the environment and coordinate actions in response. As shown in section A of Figure~\ref{fig:mesulam}, simple organisms such as frogs demonstrate a topological organization that facilitates simple processing, like autonomous responses such as the frog tongue reflex in response to a fly \cite{mesulam1998sensation}. Humans on the other hand have more complex intermediate processing that is flexible and facilitates cognitive features such as memory and thoughtful decision making \cite{mesulam1998sensation}.

\begin{figure}[!h]
    \centering
    \includegraphics[width=0.7\linewidth]{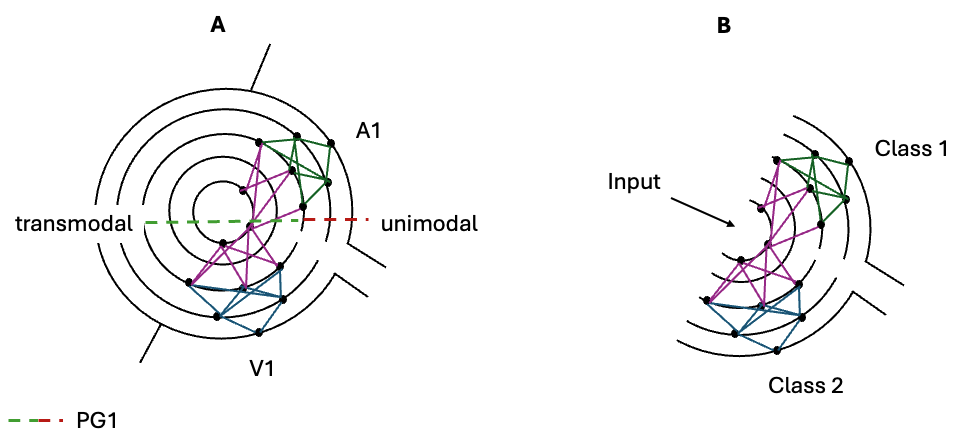}
    \caption{A: Figure adapted from \cite{mesulam1998sensation} and \cite{margulies2016situating} demonstrating the unimodal transmodal axis emergence from the topological organization of brain regions that are involved in primary sensory processing, to regions involved in higher order association. Visual input to the visual cortex (V1) is processed, independently of auditory input to A1, in unimodal regions before multi-modal association happens in transmodal regions. In B we demonstrate a similar concept, this time outlining how ANNs take in input and differentiate between classes as output. As shown, each class could be considered as a separate stimuli as output layers strive to separate learned represenations for correct classification. From this we propose understanding artificial functional connectivity better for interpreting ANNs.}
    \label{fig:mesulams_comparison}
\end{figure}

\textbf{The Unimodal Transmodal Axis In Neuroscience.}
Mesulam's work was then empirically validated by \cite{margulies2016situating}, demonstrating the structural and functional differences between regions involved in primary processing of input stimuli and regions involved in multi-modal information association. Rather than having discrete processing stages to map input to output, brain networks exhibit a hierarchy that is often conceptualized as a continuous gradient, or principle gradient (PG1) spanning from unimodal to transmodal regions \cite{margulies2016situating}. Unimodal regions are cortical areas primarily specialized for processing information from a single sensory modality and are typically characterized by strong SFC, meaning that their functional dynamics are closely constrained by underlying structural connectivity \cite{honey2010can}. Conversely, transmodal regions are cortical areas involved in integrating information across multiple modalities and supporting higher-order cognition. In these areas network dynamics are more flexible and less directly constrained by underlying structural connections \cite{margulies2016situating}. A visual overview of this concept can be seen in section A of Figure~\ref{fig:mesulams_comparison}, where input to the visual cortex (V1) is processed via V1-specific networks and input to the auditory cortex (A1) is processed separately via A1-specific networks. While processing is kept separate initially, eventually transmodal regions enable integration of information as part of higher order functional networks such as the default mode network. 

\textbf{Losing Structure Does Not Limit Functional Connectivity In Biological Networks}
Differences in structural and functional relationships across the cortex is what gives rise to the macroscale cortical hierarchy in the brain \cite{mesulam1998sensation,margulies2016situating}. Thus enabling the differences in information processing whilst remaining embedded in the constrained single structure that is the human brain. Yet even when the brain loses structural components, it is able to map sensation to action if functional connectivity is preserved. For example individuals with split-brain, a situation where the main bundle of nerve fibres connecting the left side of the brain with the right has been severed (usually to treat severe cases of epilepsy) \cite{uddin2008residual}, functional connectivity has still been observed on a task-dependent basis \cite{uddin2013complex}. Thus highlighting the notion that functional connectivity, while constrained, is not completely limited by structural connectivity. This inspires our consideration of artificial functional connectivity preservation to guide structured pruning, by removing connections that contribute little to the network’s functional dynamics; as preserving all structural links may not be necessary for the task. This provides a more interpretable and direct method of identifying a trainable sparse network through structured pruning. 

\textbf{Proposing Artificial Functional Connectivity In ANNs}
 ANNs historically draw on concepts from biological neural processing. In particular, convolutional neural networks (CNNs) derive their core architectural ideas from research on the visual cortex, where neurons exhibit localized receptive fields and hierarchical feature selectivity \cite{dapello2020simulating}. However, understanding what topological organization may occur in ANNs is still relatively underexplored \cite{songdechakraiwut2025functional,zhang2023functional}. In this work, our investigation into AFC is based on the perspective of understanding the types of intermediate processing between input and image classification. As shown in section B of Figure~\ref{fig:mesulams_comparison}, mapping input to output happens differently in CNNs as segregation happens in the output layer. So what adopting a neuroscience lens, contextualizing AFC for ANNs, could do for model compression using structured pruning is a question this work aims to contribute to. 
\section{Supplementary Diagrams For Background}
\begin{figure*}[ht] 
    \centering
    \includegraphics[width=\textwidth,keepaspectratio]{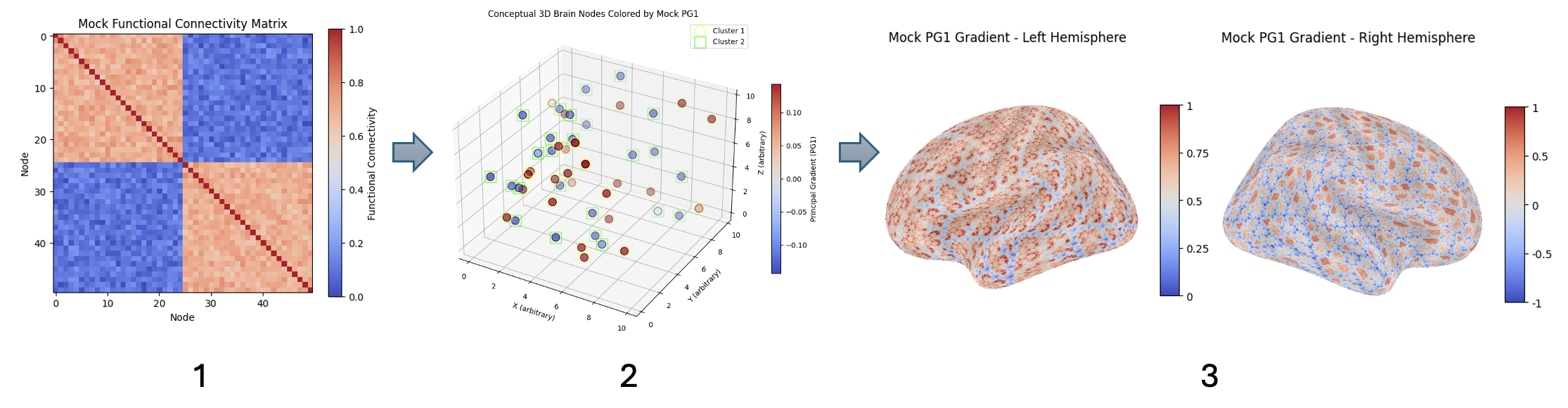}
    \caption{This figure shows a conceptual 3D “brain” (3) colored by their principal gradient (PG1) values computed from a simulated functional connectivity matrix (1). This provides a simplistic conceptual demonstration of how PG1 can reveal continuous gradients and cluster structure in a network. Section 3 shows a simulated principal gradient mapped onto the inflated cortical surfaces of the left and right hemispheres, with color indicating the gradient value across vertices, illustrating how a continuous functional gradient would appear spatially on the brain.}
    \label{fig:pg_ext_brain} 
\end{figure*}
\begin{figure*}[!h]
    \centering
    \centering
    \includegraphics[width=\linewidth]{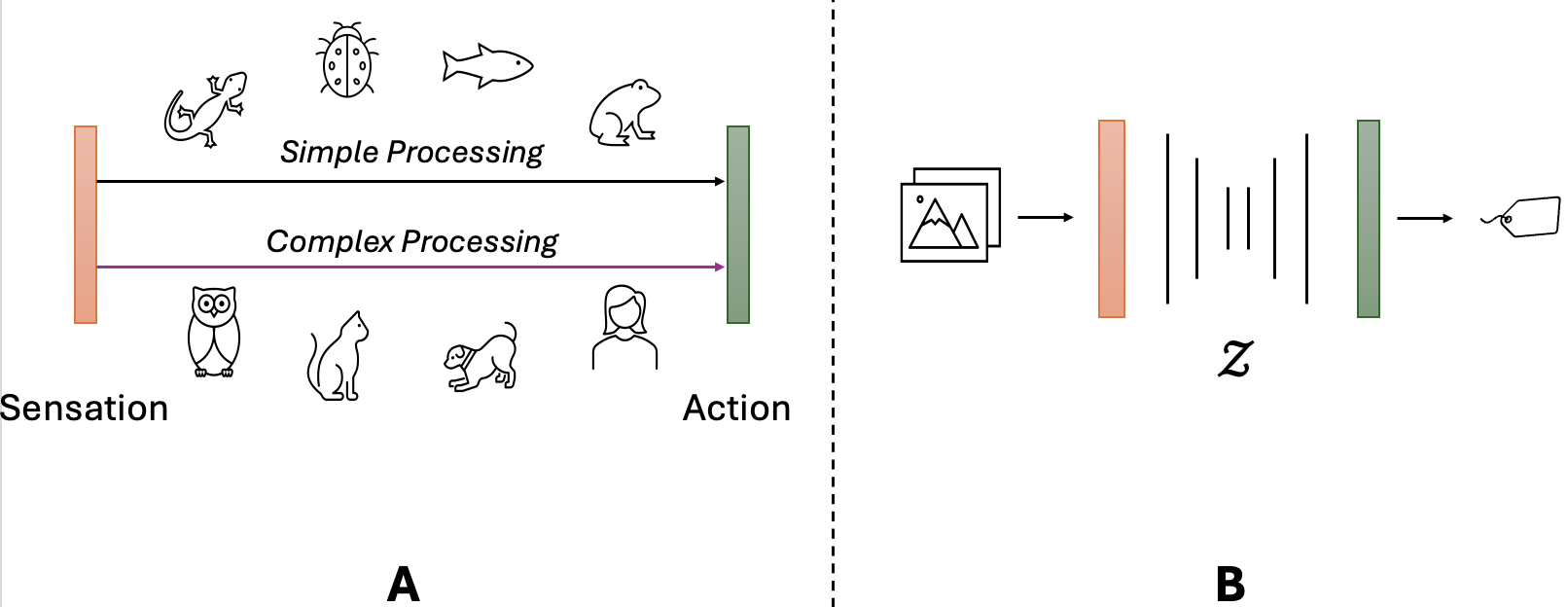}
    \caption{\textbf{A} shows Mesulam's articulation of the intermediate processing that facilitates the sensing of external stimuli to organisms taking action \cite{mesulam1998sensation}. Simple processing is typically done by less complex organisms that have autonomous responses to stimuli. Complex intermediate processing is what can be found in more complex organisms such as cats, dogs, birds and humans. \textbf{B} shows a similar interpretation we make, where \textit{Z} represents the feature representations learned by the network that enables the classifier head to map labels based on input \textit{"stimuli"}.}
    \label{fig:mesulam}
\end{figure*}
\subsection{Abbreviations and Terminology}
We include in Table~\ref{tab:abbreviations}, an overview of the key abbreviations used in this work and their definitions and meanings for clarity. 
\begin{table*}[!h]
\caption{Table of Abbreviations Used In This Work}
\label{tab:abbreviations}
\centering
\small
\setlength{\tabcolsep}{6pt}
\begin{tabular}{l p{0.85\textwidth}}
\toprule
\textbf{Abbreviation} & \textbf{Definition \& Meaning} \\
\midrule
ASF-S & \textit{Artificial Structure Function Search} --- Our Pruning Framework. \\[6pt]
AFC   & \textit{Artificial Functional Connectivity} --- The structure-function relationship that emerges in a trained neural network. \\[6pt]
PG1   & \textit{Principal Gradient} --- The first non-trivial eigenvector extracted from the output layer via diffusion embedding, representing the dominant axis of variation in the layer's activation space. \\[6pt]
PGI   & \textit{Principal Gradient Importance} --- A novel metric quantifying how each hidden unit in an ANN contributes to the PG1 contrast vector, computed as the difference between the output-layer PG1 for correct minus incorrect predictions. \\
\bottomrule
\end{tabular}
\end{table*}

\newpage
\section{Model Training and Evaluation}
Here we include details of the procedures and hyperparameters utilized to obtain the pre-trained model checkpoints used for our experiments. All models were implemented in \texttt{PyTorch} and trained on GPU acceleration where available. The training utilized the standard training splits of the MNIST and CIFAR-10 datasets. All images were normalized using dataset-specific mean ($\mu$) and standard deviation ($\sigma$) values. For MNIST, images were normalized with $\mu = 0.1307$ and $\sigma = 0.3081$. Training and validation data were loaded using a batch size of 60. For CIFAR-10, images were normalized using $\mu = [0.4914, 0.4822, 0.4465]$ and $\sigma = [0.2023, 0.1994, 0.2010]$. Training data utilized significant augmentation to enhance robustness and generalization, including random horizontal flip, random crop ($32 \times 32$ with 4-pixel padding), color jitter (brightness, contrast, saturation $0.2$; hue $0.1$), and random rotation ($15^{\circ}$). Both training and validation data were processed using a batch size of 128.

The following network architectures, as shown in Table~\ref{tab:training_params}, were trained to convergence. Training employed the Adam optimizer and utilized a ReduceLROnPlateau scheduler to dynamically adjust the learning rate based on validation loss, with a patience of 10 epochs and a reduction factor of 0.1.

\begin{table}[!h]
\centering
\captionsetup{justification=centering}
\caption{Training hyperparameters for all models.}
\label{tab:training_params}
\begin{tabular}{lcccccc}
\toprule
\textbf{Model} & \textbf{Dataset} & \textbf{Epochs (Max)} & \textbf{Initial LR} & \textbf{Weight Decay} & \textbf{Batch Size} \\
\midrule
LeNet300\_100 & MNIST & 50  & $1.2 \times 10^{-3}$ & $1 \times 10^{-4}$ & 60  \\
Conv2Net      & CIFAR-10 & 100 & $2 \times 10^{-4}$ & $5 \times 10^{-4}$ & 128 \\
Conv6Net      & CIFAR-10 & 100 & $3 \times 10^{-4}$ & $5 \times 10^{-4}$ & 128 \\
\bottomrule
\end{tabular}
\end{table}
The model  corresponding to the lowest validation loss achieved was saved as the final best-performing checkpoint for subsequent activation extraction. 

We evaluated the performance of our trained neural network models on their respective test datasets. For the CIFAR-10 models, \texttt{Conv2Net} and \texttt{Conv6Net}, we used the full CIFAR-10 test set consisting of 10,000 images. For the MNIST model, \texttt{LeNet300\_100}, we used the standard MNIST test set of 10,000 grayscale images.

Table~\ref{tab:model_eval} summarizes the accuracy, parameters and FLOPS for each model.

\begin{table}[!h]
\centering
\captionsetup{justification=centering}
\caption{Evaluation results for all trained models on their respective test datasets.}
\label{tab:model_eval}
\begin{tabular}{lccc}
\toprule
\textbf{Model Name} & \textbf{Accuracy (\%)} & \textbf{Params (K)} & \textbf{FLOPs (M)} \\
\midrule
Conv2Net      & 72.69 & 183 & 40.89 \\
Conv6Net     & 88.10 & 815 & 79.47 \\
LeNet300\_100    & 98.51 & 267 & 0.53 \\
\bottomrule
\end{tabular}
\end{table}

 \section{Proofs}\label{app:proofs}
 \begin{proposition}
     For $x$ with label $y$, let $\gamma(x) = f_y(z(x)) - \max_{i\ne y}f_i(z(x))$ be the classification margin of the unpruned network, $B = \max_i\|w_i^{(f)}\|$, and $\varepsilon(\tau) = \|z(x)-\tilde z(x)\|$ the representation shift induced by pruning at threshold $\tau$. If $\varepsilon(\tau) < \gamma(x)/(2B)$, the unmodified output layer classifies $x$ correctly on $\tilde z(x)$.
 \end{proposition}

\begin{proof}
Each output transformation $f_i$ is affine and $B$-Lipschitz in $z$, so $|f_i(z(x))-f_i(\tilde z(x))|\le B\varepsilon(\tau)$ for all $i$. Hence $f_y(\tilde z(x)) - \max_{i\ne y}f_i(\tilde z(x)) \ge \gamma(x) - 2B\varepsilon(\tau)$, which is positive under the stated condition.
\end{proof}
\begin{assumption}\label{assump:2}
Let $u_k$ be the $k$-th PCA component of the contrast matrix $\Delta G$ (Eq.\ref{eq:cont_matrix},\ref{eq:pca}) and $\bar
w^{(f)}$ the mean output-layer weight vector. A unit $i$ in layer $l$ is \emph{task-relevant} if
there exists $\kappa_i > 0$ such that
\begin{equation}
    \big\langle w_i^{(l)},\, u_k \odot \bar w^{(f)} - \bar w^{(f)} \big\rangle \;\ge\; \kappa_i.
\end{equation}
\end{assumption}
\begin{proposition}
    Under Assumption \ref{assump:1}, PGI-based scoring assigns greater importance to units that are structurally aligned with the network's task-discriminative behavior. Formally, for every task-relevant unit $i$ in layer $l$, 
\begin{equation}
   s_{i,k}^{(l)} = \big\langle w_i^{(l)}, u_k\odot\bar w^{(f)}\big\rangle \ge \big\langle w_i^{(l)}, \bar w^{(f)}\big\rangle = p_{i,k}^{(l)}.
\end{equation}
\end{proposition}
\begin{proof} 
By definition, $s_{i,k}^{(l)} = \langle w_i^{(l)}, u_k \odot \bar w^{(f)} \rangle$
and $p_{i,k}^{(l)} = \langle w_i^{(l)}, \bar w^{(f)} \rangle$. By linearity of the inner product,
\[
s_{i,k}^{(l)} - p_{i,k}^{(l)} = \big\langle w_i^{(l)},\, u_k \odot \bar w^{(f)} - \bar w^{(f)}
\big\rangle.
\]
For task-relevant units, this is at least $\kappa_i > 0$ by Assumption~1, hence $s_{i,k}^{(l)} \ge
p_{i,k}^{(l)}$. 
\end{proof}

\section{Additional Results \& Discussion}
\subsection{Principle Gradient Importance}
Here we include the normalized PGI scores extracted per layer for each model. 
As shown in Figure~\ref{fig:lenet_pgi}, each neurone has its own level of importance for contributing to the preservation of organization in the output layer when classification of a MNIST digit is done correctly. There are many neurons that do not exceed an PGI score of 0.2, and pruning reduces up to 62\% of parameters from this low threshold. Even thresholds up to 0.8 which led to 97\% reductions in parameters, maintained accuracy within 4\% of the original. For CNNs with multiple layers to prune (as in Figure~\ref{fig:conv2net_pgi} and \ref{fig:conv6net_pgi}), we are able to get  PGI scores for each layer using the PCA components from the PG1 contrast vector. This is important as we observe the extraction of multiple pruned model variants for Conv2Net and Conv6Net, where we are able to reach the same sparsity levels but with some models yielding better accuracies (within 2\% of base model accuracy as opposed to within 5\%). The only difference in method of yielding these variants, was the use of different thresholds combinations during ASF-S.  

\begin{figure*}[!h] 
    \centering
    \includegraphics[width=0.8\textwidth,keepaspectratio]{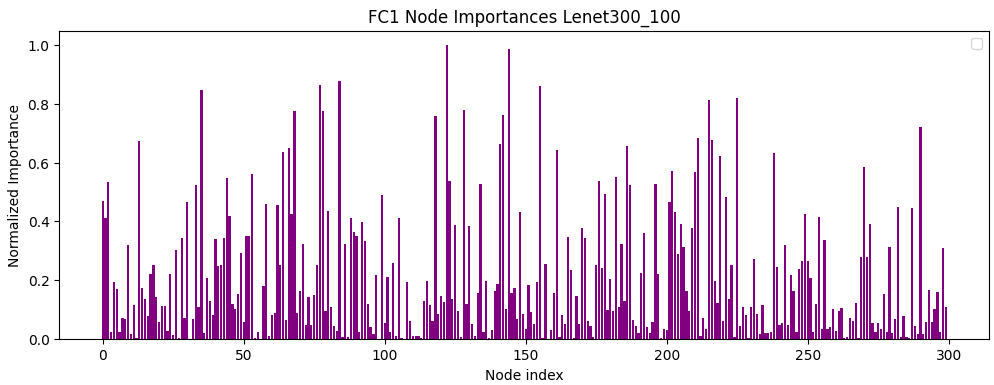}
    \caption{Normalized neuron importance for FC1 for LeNet. Filters with lower PGI scores served as candidates for pruning.}
    \label{fig:lenet_pgi}
\end{figure*}
\begin{figure*}[!h] 
    \centering
    \includegraphics[width=0.8\textwidth,keepaspectratio]{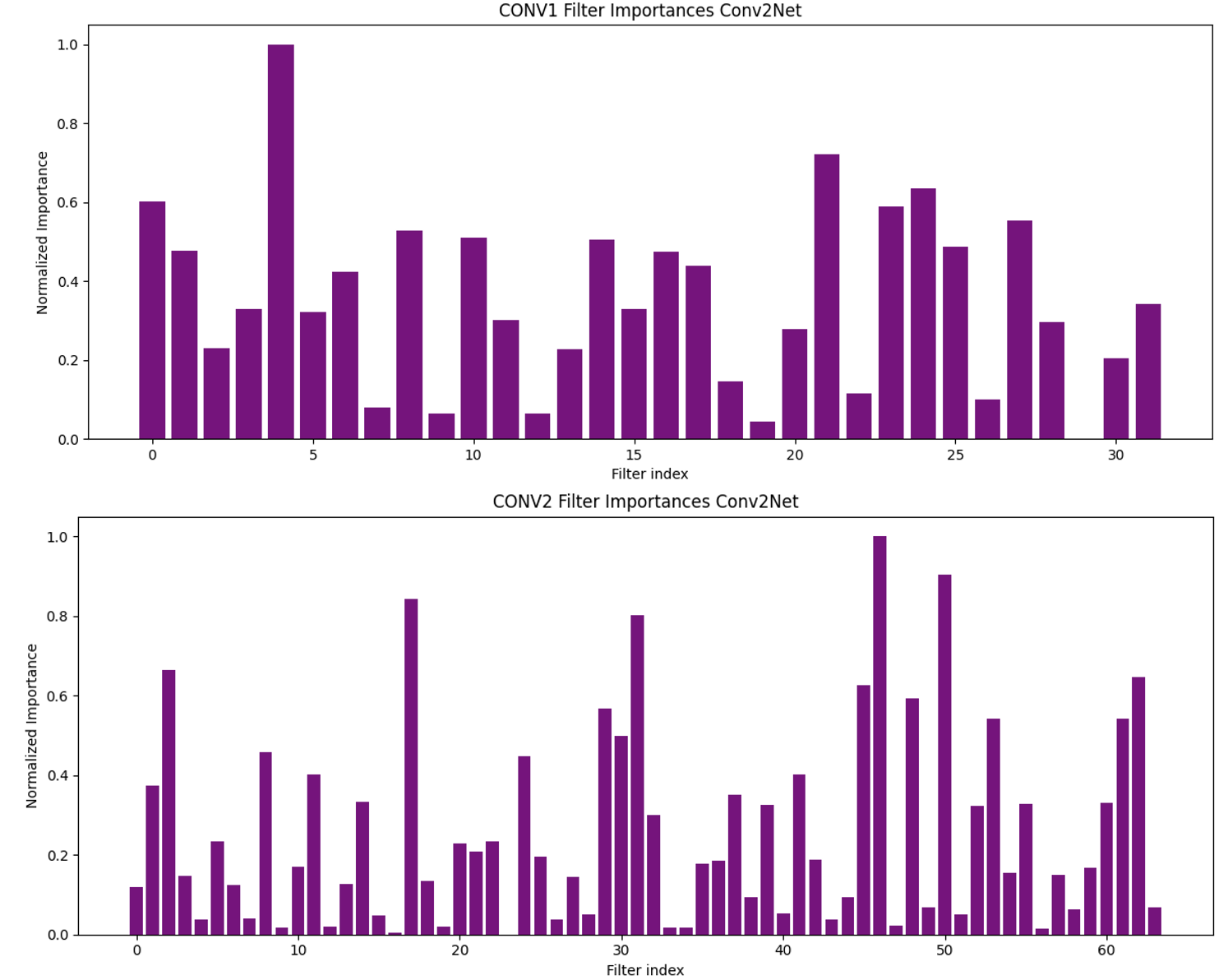}
    \caption{Normalized Filter importance for Conv1 and Conv2 for Conv2Net. Filters with lower PGI scores served as candidates for pruning.}
    \label{fig:conv2net_pgi}
\end{figure*}
\begin{figure*}[!h] 
    \centering
    \includegraphics[width=0.8\textwidth,keepaspectratio]{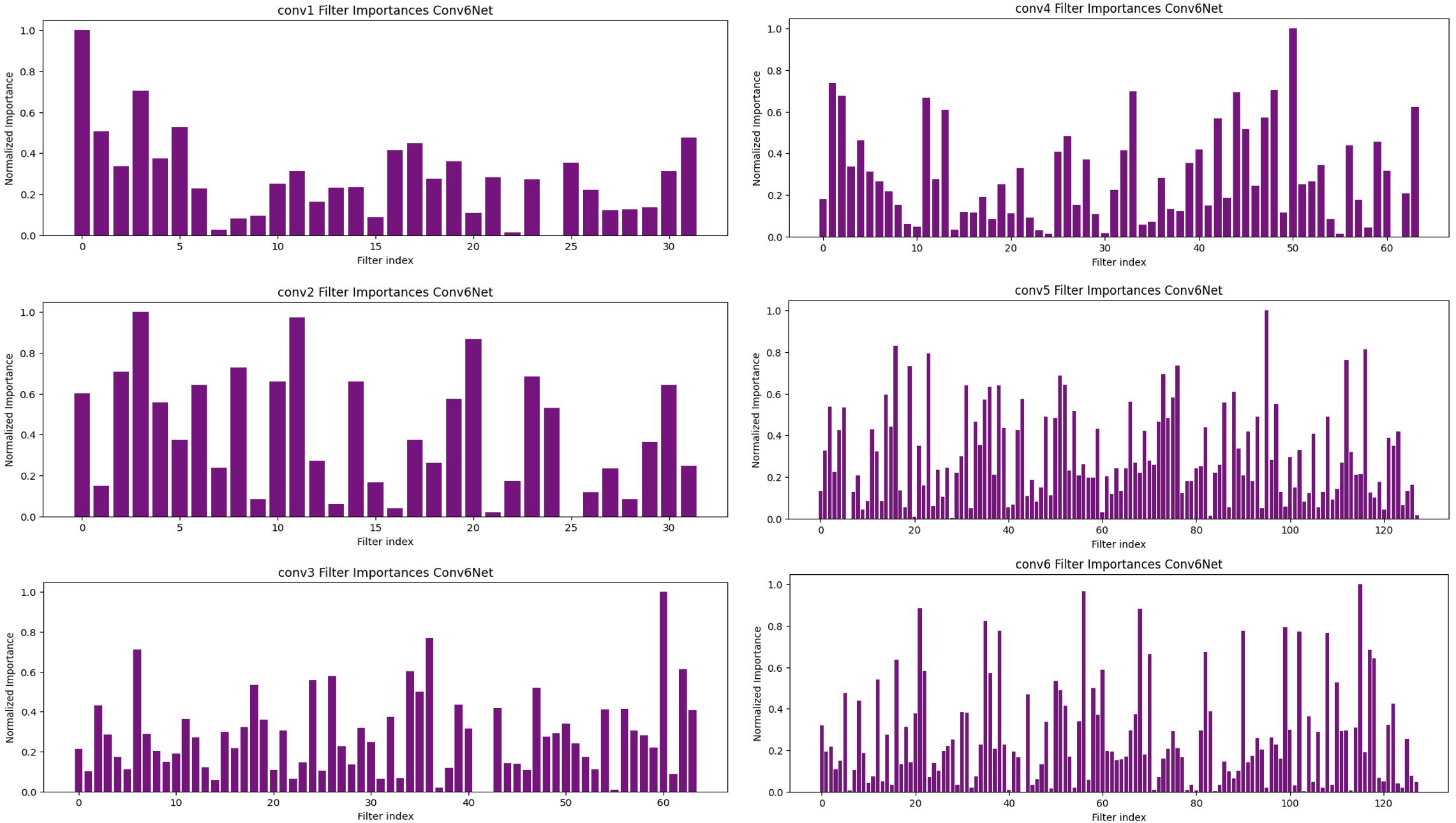}
    \caption{Normalized Filter importance for Conv1...Conv6 for Conv6Net. Filters with lower PGI scores served as candidates for pruning.}
    \label{fig:conv6net_pgi}
\end{figure*}

\begin{table*}[h]
\centering
\caption{Top FLOPs reduction per pruned model variant $\times$ Method (ASF-S + other pruning techniques)}
\label{tab:top_FLOPs_clean}
\begin{tabular}{llrrr} 
\hline
\textbf{Pruned Model Variant} & \textbf{Method} & \textbf{$\Delta$Acc} & \textbf{Params Reduction (\%)} & \textbf{FLOPs Reduction (\%)} \\
\hline
conv2 & topk   & -16.12 & 91.5 & 98.5 \\
conv2 & random & -10.64 & 87.0 & 96.3 \\
conv2 & L1 & -12.99&90.0&97.8\\
conv2 & L2 & -13.74&90.00&97.8\\
conv2 & \textbf{ASF-S} & -\textbf{3.78}  & 75.8 & \textbf{82.7} \\
conv6 & random & -28.55 & 95.3 & 99.0 \\
conv6 & topk   & -25.11 & 93.3 & 98.7 \\
conv6 & L1 & -24.26&92.7&98.2\\
conv6 & L2 & -21.05&92.7&98.2\\
conv6 & \textbf{ASF-S} & \textbf{-3.97 } & 67.0 & \textbf{75.5} \\
lenet & top-k    & -2.08  & 90.0 & 89.9 \\
lenet & l2     & -1.80  & 89.6 & 89.6 \\
lenet & \textbf{ASF-S }& \textbf{-2.89 } & 94.9 & \textbf{94.9} \\
\hline 
\end{tabular}
\end{table*}
\newpage
\subsection{ASF-S Preserves Accuracy Whilst Reducing FLOPs}
We applied ASF-S guided by PGI scores derived from our PG1 multiple times for different threshold combinations for layers of \textit{Conv2Net}, \textit{Conv6Net} and examined a range of thresholds from 0.1 to 0.9 for \textit{LeNet300\_100}. For convolutional layers, PGI scores quantify each filter's contribution to class-wise correct versus incorrect activation patterns, while for fully connected layers, neurons are scored instead. Filters or neurons with PGI below layer-specific thresholds were pruned in a single step, followed by fine-tuning either on fully connected layers alone or the full network. We were able to demonstrate that our method achieves much better preservation of accuracy than the other methods we experimented with namely  $\ell_{1}$-norm pruning, $\ell_{2}$-norm pruning, top-$K$ pruning, and random pruning. ASF-S yielded only between 3-4\% drops in accuracy, whilst enabling up to 83\% reduction in FLOPs in CNNs and 95\% reduction in FLOPs of MLPs. This suggests ASF-S can identify a sparse networks\cite{frankle2019lottery,frantar2023sparsegpt} using structured pruning for model compression with preservation of accuracy. 



 \newpage
\subsection{Functional Connectivity Maps}
Here we discuss our results for extracting the functional connectivity maps from the models investigated in this work. In line with works in RSA and functional networks \cite{kornblith2019similarity,kriegeskorte2008representational,zhang2023functional}, we found, upon visual inspection, differences between correct and incorrect classification of all classes. Our results can be seen in Figure~\ref{fig:fc_maps_conv}. The topological organization can be inferred from what we see demonstrated in these AFC maps. Despite the same dataset and training regimes being used we can see that different architectures learn different topographies in the output layer for each class. This difference across classes is also shown in the contrast vectors shown in Figure~\ref{fig:contrast_vector_conv2_1}

Additionally, when the model incorrectly classifies, we see that the topographical organization is also different, with certain units exhibiting differences in correlated activity. Another interesting finding is the patterns that emerge for living organisms versus non-living objects. In A1 for living organisms (2,3,4,5,6,7) a small square can be seen in the top right hand corner indicating similar co-activation of units across samples within that class. This square is not present in non-living objects. Similar visual similarities for the discrepancy between living and non-living can be seen in row B1 also. This highlights the that AFC between units that drives the final interpretation of the learned representations by the hidden layers in CNNs, is useful for discriminating between classes as well as correct and incorrect classification.

\begin{figure}[!h]
    \centering
    \includegraphics[width=\linewidth]{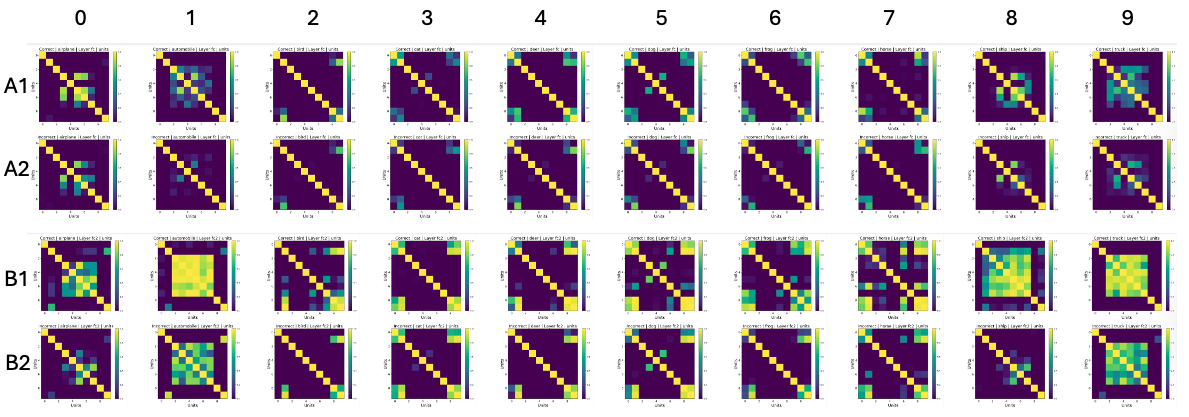}
    \caption{This shows the artificial functional connectivity maps for Conv2Net: Row A1 shows the maps from when the model correctly classified, columns 1-9 show the classes airplane, automobile, bird, cat, deer, dog, frog, horse, ship and truck respectively. Row A2 shows the same but for incorrect classification. Row B1 shows the maps for Conv6Net when classification was done correctly, row B2 shows same but for incorrect classification.} 
    \label{fig:fc_maps_conv}
\end{figure}

\begin{figure}[!h]
    \centering
    \includegraphics[width=\linewidth]{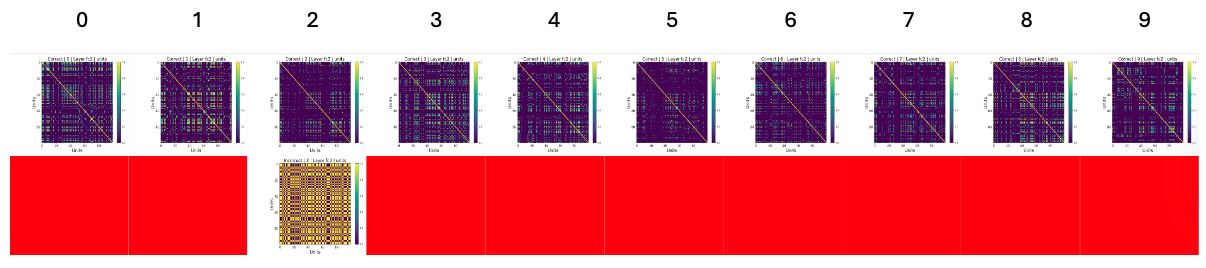}
    \caption{Shows the artificial functional connectivity maps for LeNet, the column headings correspond to the digit in the MNIST dataset. The top row shows the maps for correct classification and the bottom row shows the map for incorrect classification.}
    \label{fig:fc_maps_lenet}
\end{figure}
\newpage
\subsection{Results For Principle Gradient Vector Extraction}
Here we provide a discussion around our interpretation of the extracted PG1 per layer. In neuroscience PG1 demonstrates a single axis to differentiate between clusters of regions that may be part of unimodal processing or transmodal processing. But in our work we are able to provide some preliminary analysis about class discrimination and incorrect versus correct conditions. 

\begin{figure}[!h]
    \centering
    \includegraphics[width=\linewidth]{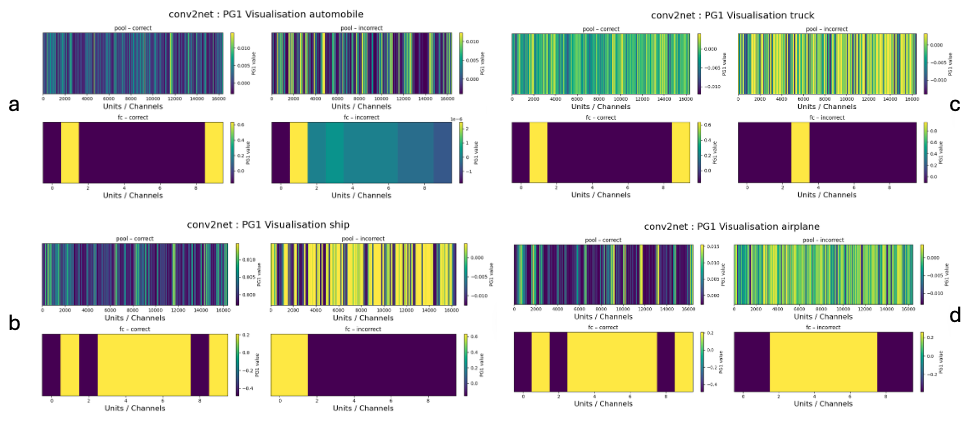}
    \caption{Here we include the PG1 vectors extracted from Conv2Net for CIFAR-10 classes that we deem semantically similar as they are all non-living objects. a shows results for the automobile class: top left shows the PG1 for the maxpool layer, bottom left shows the PG1 for the output layer when classification was correct, the right side shows the same but for the incorrect condition. b: shows results for the results for the ship class, top left shows the PG1 for the maxpool layer, bottom left shows the PG1 for the output layer when classification was correct, the right side shows the same but for the incorrect condition. c: shows the results for the truck, top left shows the PG1 for the maxpool layer, bottom left shows the PG1 for the output layer when classification was correct, the right side shows the same but for the incorrect condition. d: shows results for the airplane class, top left shows the PG1 for the maxpool layer, bottom left shows the PG1 for the output layer when classification was correct, the right side shows the same but for the incorrect condition.}
    \label{fig:pg_conv2_1}
\end{figure}

Visualization of the PG1 values across the output layer reveals that each class has distinct topographical structures, with unique clustering of neuroes driving correct classification. Some classes even exhibit exclusively positive values in the PG1 vector, while others span negative to slightly positive values. This suggests that AFC exhibited by a single network can be diverse across classes as well as condition (correct versus incorrect). In the pooling layers, PG1 shows a broader range of values and less polarity, suggesting a richer and more distributed topographical organization. In neuroscience, PG1 led to evidence for the existence of large scale functional networks \cite{mesulam1998sensation,margulies2016situating,Yeo2011_7Networks} that make up the intermediate processing mapping sensation to action. Our results complement the notion that similar large scale networks, making up AFC, could exist in ANNs as well \cite{zhang2023functional,songdechakraiwut2025functional}.  Our results can be seen in Figures~\ref{fig:pg_conv2_1}, \ref{fig:pg_conv2_2},\ref{fig:pg_conv6_1}, \ref{fig:pg_conv6_2} and \ref{fig:pg_lenet} for Conv2Net, Conv6Net and LeNet respectively.

\begin{figure}[!h]
    \centering
    \includegraphics[width=\linewidth]{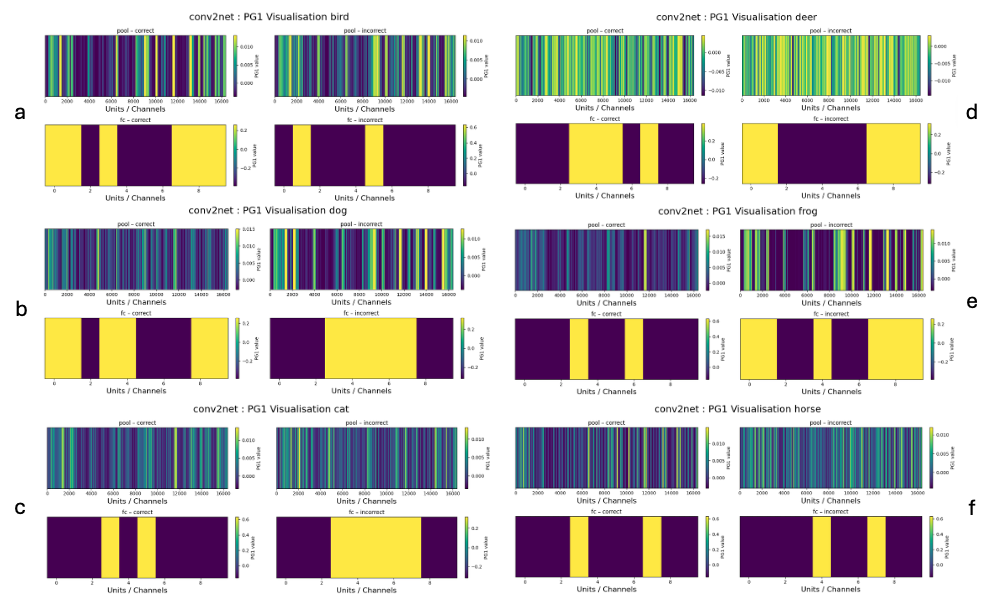}
    \caption{Here we include the PG1 vectors extracted from Conv2Net for CIFAR-10 classes that we deem semantically similar as they are all living organisms. a: shows results for the bird class,  top left shows the PG1 for the maxpool layer, bottom left shows the PG1 for the output layer when classification was correct, the right side shows the same but for the incorrect condition. b: shows results for the dog class,  top left shows the PG1 for the maxpool layer, bottom left shows the PG1 for the output layer when classification was correct, the right side shows the same but for the incorrect condition. c: shows results for the cat class,  top left shows the PG1 for the maxpool layer, bottom left shows the PG1 for the output layer when classification was correct, the right side shows the same but for the incorrect condition. d: shows results for the deer class,  top left shows the PG1 for the maxpool layer, bottom left shows the PG1 for the output layer when classification was correct, the right side shows the same but for the incorrect condition. e: shows results for the frog class,  top left shows the PG1 for the maxpool layer, bottom left shows the PG1 for the output layer when classification was correct, the right side shows the same but for the incorrect condition. f: shows results for the horse class, top left shows the PG1 for the maxpool layer, bottom left shows the PG1 for the output layer when classification was correct, the right side shows the same but for the incorrect condition.}
    \label{fig:pg_conv2_2}
\end{figure}
In interpreting PG1 across network layers, it is also important to distinguish between feature-rich hidden layers and the label-specific output layer. In hidden layers, units with similar PG1 values could be contributing to similar computations, analogous to the unimodal-transmodal axis in the brain. In contrast, in the output layer, units are directly aligned with class decisions, such that PG1 primarily reflects the dominant decision axis rather than distributed feature structure. Consequently, while PG1 could reveal functional networks in hidden layers, its interpretation in the output layer is limited to class-specific utilization of a shared decision axis. Therefore using PG1 of the output layer to guide pruning enables us to identify which units in the hidden layers are key for conducting the relevant computations needed for the task the ANN was trained on. This is the basis on which we created out PGI scoring system. 
\begin{figure}[!h]
    \centering
    \includegraphics[width=\linewidth]{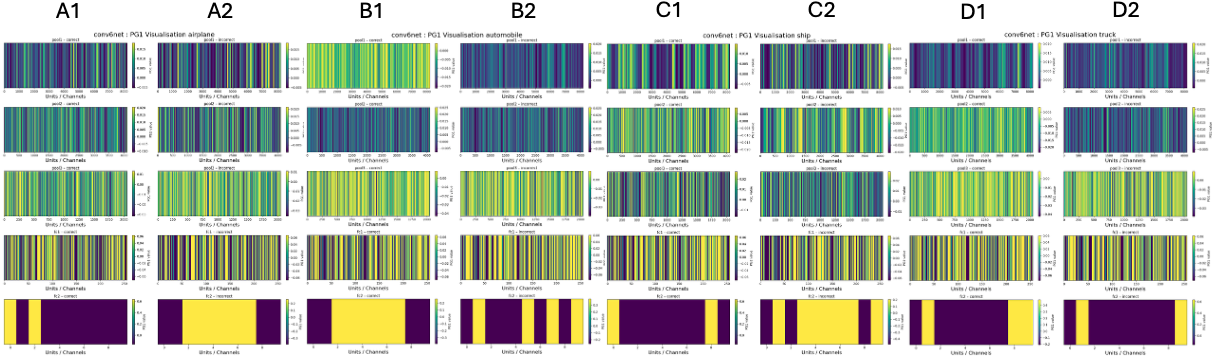}
    \caption{Here we include the PG1 vectors extracted from Conv6Net for CIFAR-10 classes that we deem semantically similar as they are all non-living objects. A1 shows the correct classification of the airplane, A2 shows incorrect. B1 shows correct classification of the automobile, B2 shows incorrect. C1 shows correct classifcation of the ship, C2 shows incorrect. D1 shows correct classification of the truck, D2 shows incorrect.}
    \label{fig:pg_conv6_1}
\end{figure}

\begin{figure}[!h]
    \centering
    \includegraphics[width=\linewidth]{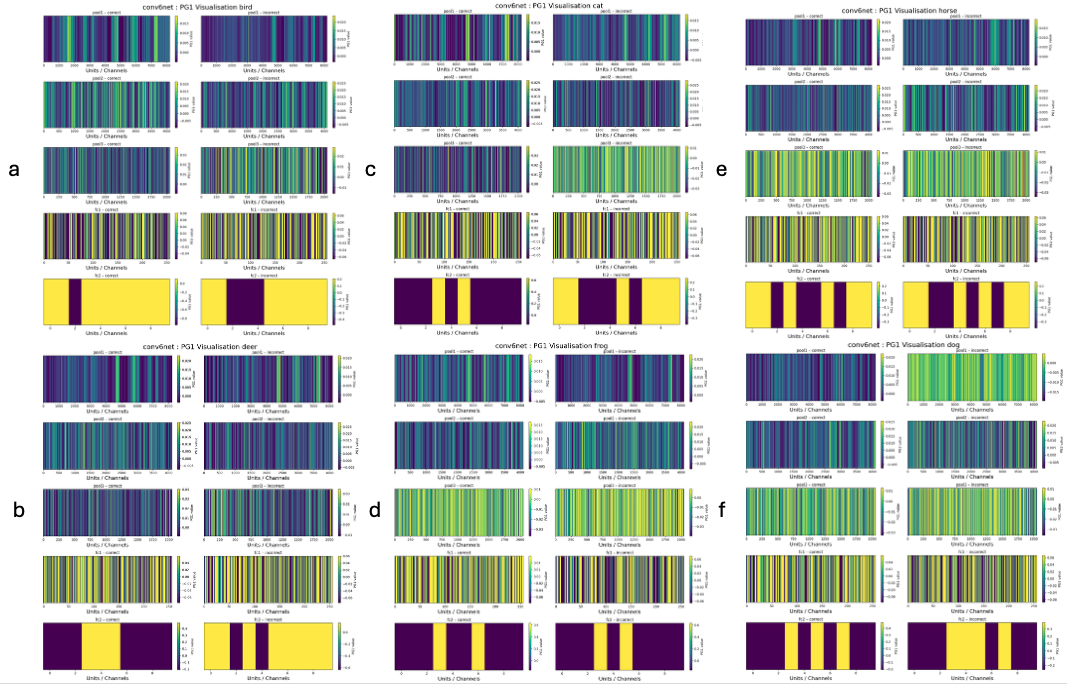}
    \caption{Here we include the PG1 vectors extracted from Conv6Net for CIFAR-10 classes that we deem semantically similar as they are living organisms. From a to f, we show results for(a) bird, (b) deer, (c) cat, (d) frog, (e) horse and (f) dog. For all classes the left-hand column for each image represents the PG1 for the correct vector, the right-hand side is incorrect cases.}
    \label{fig:pg_conv6_2}
\end{figure}

\begin{figure}[!h]
    \centering
    \includegraphics[width=\linewidth]{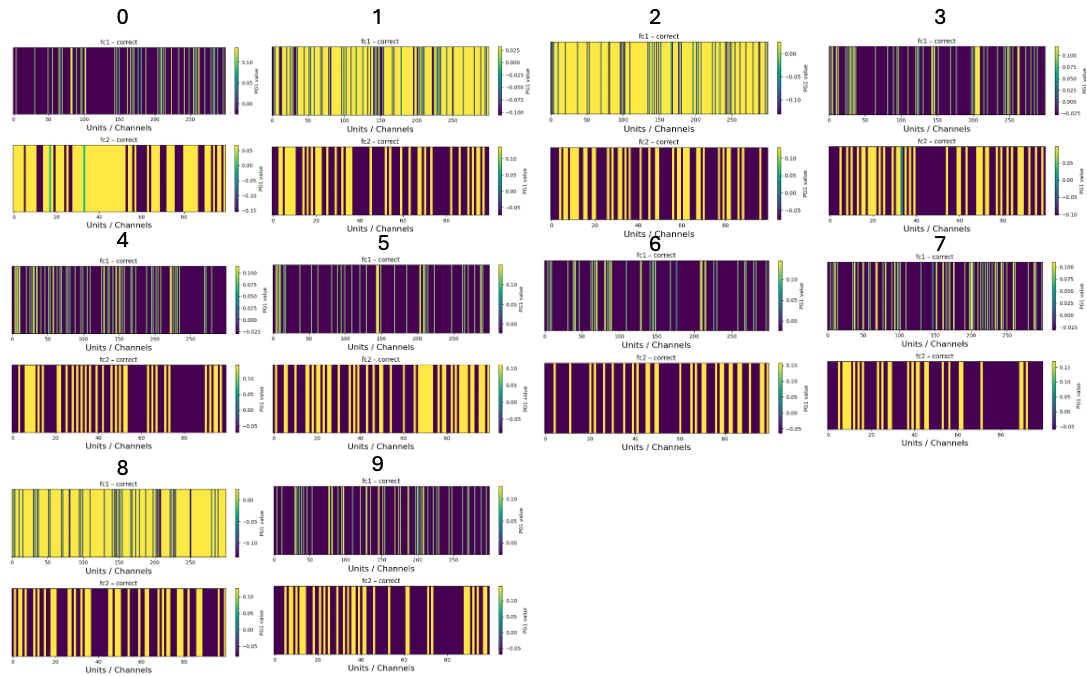}
    \caption{Here we include the PG1 vectors extracted from Lenet for MNIST digits 0 to 9. All results here demonstrate correct classification, the top image for each class is for the first FC layer and the bottom image corresponds to the second FC layer.}
    \label{fig:pg_lenet}
\end{figure}

\clearpage
\subsection{Contrast Vector Extraction}
Due to differences in PG1 vectors across classes, we wanted to ensure that our PGI importance scores were calculated in a way that ensured preservation of units that contribute to retaining all classes. The condition to satisfy our pruning was simply the difference between incorrect and correct classification, therefore our PGI extraction logic included using all contrast vectors across classes and using PCA to extract the top 4 components that accounted for $\sim$90\% of the variation. Such is what returned the PGI scores that we used to prune layers, to ensure that pruning was done in a way that identified the artificial default mode network that could preserve the best learned features for all classes. We demonstrate in Figure~\ref{fig:contrast_vector_conv2_1}, \ref{fig:contrast_vector_conv2_2}, \ref{fig:contrast_vector_conv6_1}, and \ref{fig:contrast_vector_conv6_2} our normalized contrast vectors for the output layer of Conv2Net, Conv6Net and LeNet respetively. If we interpret the PG1 contrast as a signal for determining behaviour of the output layer, we can see that it is unique to each class. This poses the potential to guide pruning regimes that could be class-specific by modifying the contrast vector logic that is used in ASF-S. That is to find the minimally viable network needed for a network to successfully classify a given class. 

\begin{figure}[!h]
    \centering
    \includegraphics[width=\linewidth]{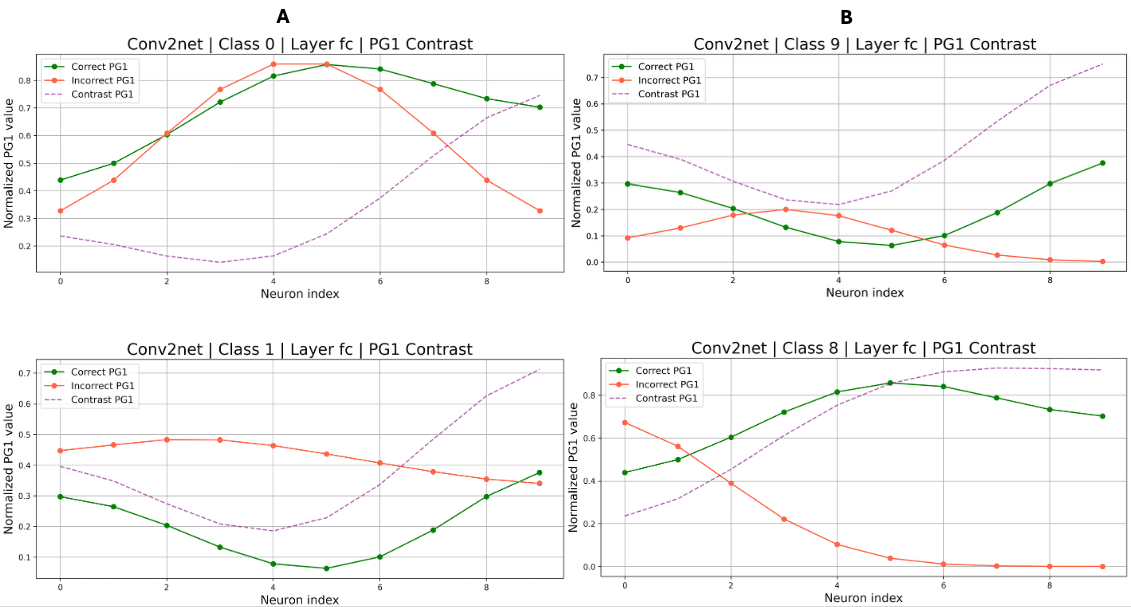}
    \caption{Contrast vectors extracted from Conv2Net for "non-organic" classes of airplane (0), automobile (1), ship (8) and truck (9).}
    \label{fig:contrast_vector_conv2_1}
\end{figure}
\begin{figure}[!h]
    \centering
    \includegraphics[width=\linewidth]{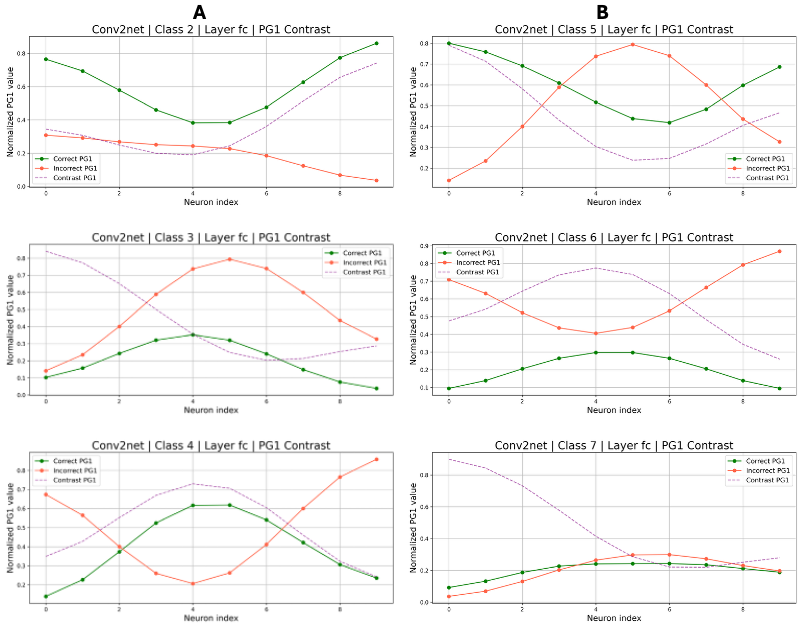}
    \caption{Contrast vectors extracted from Conv2Net for "organic" classess of bird (2), cat (3), deer (4), dog (5), frog (6) and horse (7). }
    \label{fig:contrast_vector_conv2_2}
\end{figure}

\begin{figure}[!h]
    \centering
    \includegraphics[width=\linewidth]{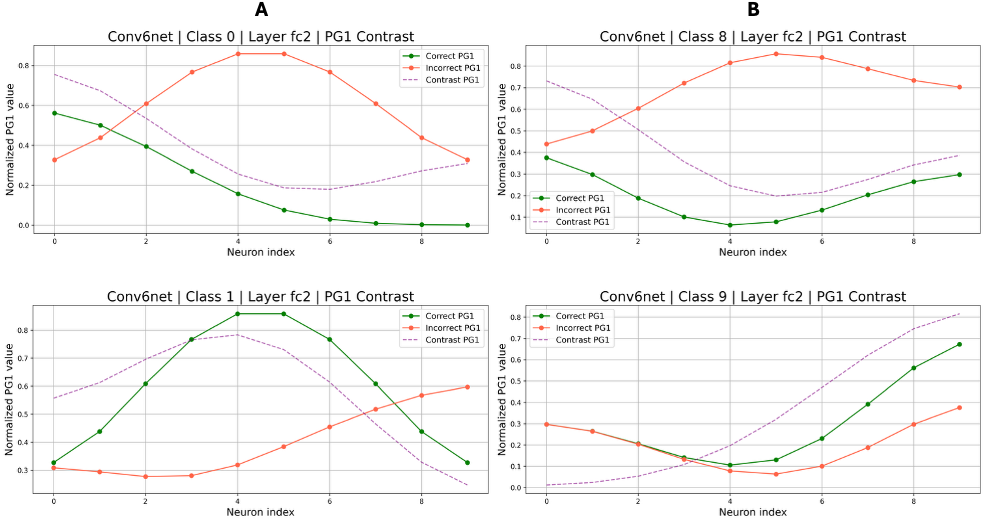}
    \caption{Contrast vectors extracted from Conv6Net for "non-organic" classes of airplane (0), automobile (1), ship (8) and truck (9).}
    \label{fig:contrast_vector_conv6_1}
\end{figure}

\begin{figure}[!h]
    \centering
    \includegraphics[width=\linewidth]{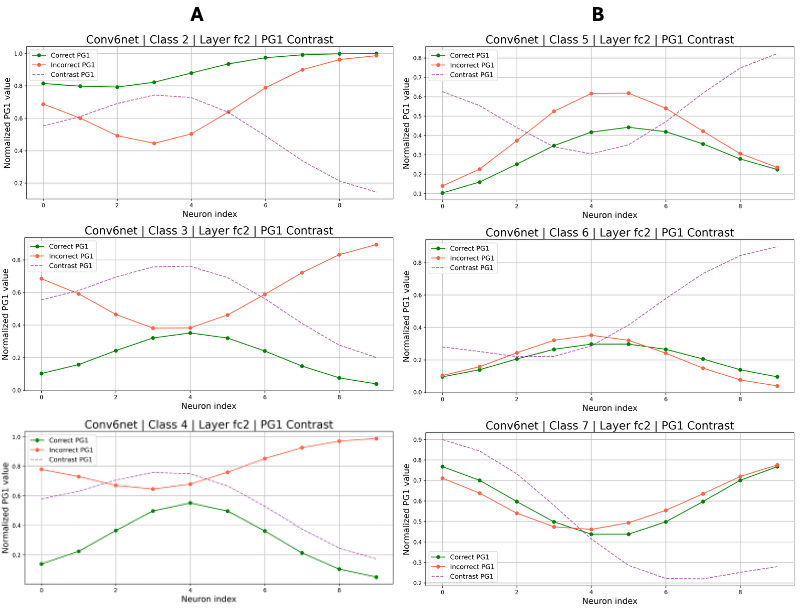}
    \caption{Contrast vectors extracted from Conv6Net for "organic" classess of bird (2), cat (3), deer (4), dog (5), frog (6) and horse (7).}
    \label{fig:contrast_vector_conv6_2}
\end{figure}


\clearpage

\subsection{Results of ASF-S on A Transformer Architecture}
We also conducted experiments on a transformer architecture.
\paragraph{Pruning and compression in Transformers/LLMs.}
Transformers are notoriously over-parameterised, and their compute/memory footprint is dominated by the feed-forward sublayers (FFNs), making them a primary target for structured compression. Prior work has explored sparsification by fine-tuning, showing that pruning can be framed as learning sparse masks (or sparse reparameterisations) while preserving downstream performance under a fixed fine-tuning budget. Complementary approaches reduce Transformer capacity via structured regularisation during training, e.g., structured layer dropping that enables depth reduction at inference time without architectural surgery. These directions broadly align with the lottery-ticket perspective: that suitably selected subnetworks can match dense-model performance, but the key challenge is \emph{how} to select structure in a way that avoids representational damage, especially under structured pruning constraints.


\paragraph{Extended BERT analysis.}
We provide additional analyses to characterise trade-offs and mechanisms. First, we report full sweep tables across sparsity levels (10--90\%) and methods, including pre- and post-prune metrics and the best head-only recovery point. Second, we plot Pareto fronts in the (Accuracy, ECE) plane to visualise deployment-relevant operating points; pruning can improve or degrade calibration depending on the method and sparsity regime. Third, we include per-layer class-conditional PG contrast curves (correct--incorrect) to illustrate that PG-guided importance is not uniform across layers, suggesting structured dependencies that are not captured by magnitude criteria.



\begin{table}[t]
\caption{\textbf{BERT head-only FFN pruning on AG News.} We report the dense
baseline, post-prune (no recovery), and best head-only recovery after pruning
at 70\% and 90\% FFN neuron removal. Metrics: accuracy (Acc) and expected
calibration error (ECE).}
\label{tab:bert_headline}
\centering
\small
\setlength{\tabcolsep}{5.5pt}
\begin{tabular}{l c cc cc}
\toprule
\textbf{Method} & \textbf{Pruned (\%)} &
\multicolumn{2}{c}{\textbf{Post-prune}} &
\multicolumn{2}{c}{\textbf{Head-only retrain}} \\
\cmidrule(lr){3-4}\cmidrule(lr){5-6}
& & Acc $\uparrow$ & ECE $\downarrow$ & Acc $\uparrow$ & ECE $\downarrow$ \\
\midrule
Baseline (dense) & 0 & \multicolumn{2}{c}{---} & 0.9454 & 0.0118 \\
\midrule
\multirow{4}{*}
 & & & & & \\[-10pt]
ASF-S (PG) & 70 & \textbf{0.9433} & 0.0106 & \textbf{0.9459} & 0.0196 \\
Random     & 70 & 0.9417          & 0.0057 & 0.9439          & 0.0178 \\
L1         & 70 & 0.9428          & 0.0042 & 0.9443          & 0.0177 \\
L2         & 70 & 0.9421          & 0.0038 & 0.9438          & 0.0185 \\
\midrule
ASF-S (PG) & 90 & \textbf{0.9400} & 0.0061 & \textbf{0.9425} & 0.0208 \\
Random     & 90 & 0.9378          & 0.0176 & 0.9400          & 0.0168 \\
L1         & 90 & 0.9420          & 0.0090 & 0.9428          & 0.0178 \\
L2         & 90 & 0.9416          & 0.0081 & 0.9432          & 0.0168 \\
\bottomrule
\end{tabular}
\end{table}

\begin{figure}
    \centering
    \includegraphics[width=0.8\linewidth]{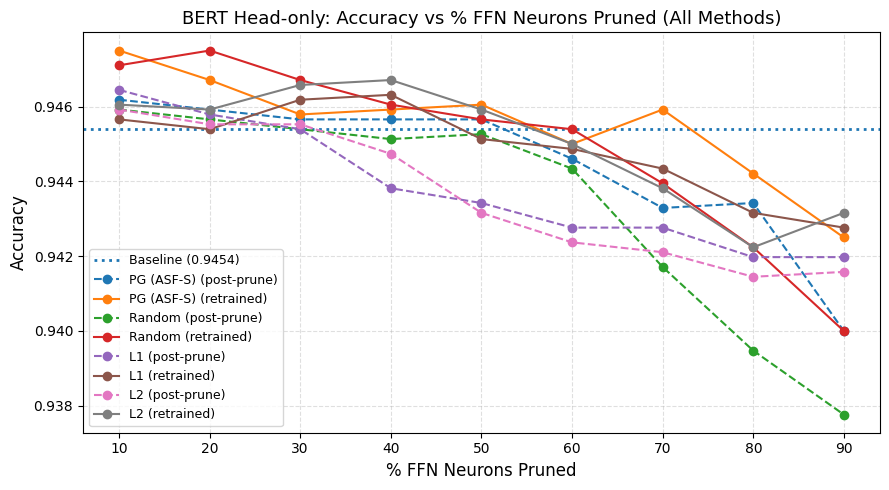}
    \caption{\textbf{Head-only BERT FFN pruning on AG News.}
    Classification accuracy as a function of the fraction of feed-forward (FFN) neurons removed from selected Transformer layers.
    We compare ASF-S (PG-guided) pruning against random, L1-, and L2-based structured pruning.
    Dashed curves denote post-pruning performance without retraining, while solid curves show recovery after fine-tuning only the task head.
    The dotted horizontal line indicates the dense baseline accuracy.}
    \label{fig:placeholder}
\end{figure}

\begin{table*}[!h]
\caption{Head-only structured FFN pruning on BERT-base (AG News). We prune FFN neurons in layers 0/5/11 and fine-tune only the task head. We report best retrained accuracy/ECE at each pruning level, and post-prune (no FT) accuracy for reference. Baseline accuracy/ECE: 0.9454 / 0.0118.}
\centering
\setlength{\tabcolsep}{4pt}
\begin{tabular}{l c cc cc cc cc}
\toprule
& & \multicolumn{2}{c}{\textbf{PG (ASF-S)}} & \multicolumn{2}{c}{\textbf{Random}} & \multicolumn{2}{c}{\textbf{L1}} & \multicolumn{2}{c}{\textbf{L2}} \\
\cmidrule(lr){3-4}\cmidrule(lr){5-6}\cmidrule(lr){7-8}\cmidrule(lr){9-10}
\textbf{Setting} & \textbf{Pruned (\%)} 
& \textbf{Acc} $\uparrow$ & \textbf{ECE} $\downarrow$
& \textbf{Acc} $\uparrow$ & \textbf{ECE} $\downarrow$
& \textbf{Acc} $\uparrow$ & \textbf{ECE} $\downarrow$
& \textbf{Acc} $\uparrow$ & \textbf{ECE} $\downarrow$ \\
\midrule
\multicolumn{10}{l}{\textbf{Baseline (no pruning)}} \\
\midrule
Head-only BERT & 0 
& \multicolumn{2}{c}{0.9454 / 0.0118} 
& \multicolumn{2}{c}{0.9454 / 0.0118}
& \multicolumn{2}{c}{0.9454 / 0.0118}
& \multicolumn{2}{c}{0.9454 / 0.0118} \\
\midrule
\multicolumn{10}{l}{\textbf{After pruning + head-only fine-tuning (best retrain)}} \\
\midrule
& 10.0  & \textbf{0.9475} & 0.0201 & 0.9471 & 0.0202 & 0.9457 & 0.0211 & 0.9461 & 0.0209 \\
& 20.0  & \textbf{0.9467} & 0.0233 & \textbf{0.9475} & 0.0214 & 0.9454 & 0.0212 & 0.9459 & 0.0216 \\
& 30.0  & 0.9458 & 0.0207 & \textbf{0.9467} & 0.0194 & 0.9462 & 0.0198 & 0.9466 & 0.0196 \\
& 40.0  & 0.9459 & 0.0213 & 0.9461 & 0.0198 & 0.9463 & \textbf{0.0174} & \textbf{0.9467} & 0.0186 \\
& 50.0  & \textbf{0.9461} & 0.0243 & 0.9457 & \textbf{0.0180} & 0.9451 & 0.0195 & \textbf{0.9459} & 0.0188 \\
& 60.0  & \textbf{0.9450} & 0.0211 & 0.9454 & 0.0183 & 0.9449 & \textbf{0.0175} & \textbf{0.9450} & 0.0189 \\
& 70.0  & \textbf{0.9459} & 0.0196 & 0.9439 & 0.0178 & 0.9443 & \textbf{0.0177} & 0.9438 & 0.0185 \\
& 80.0  & \textbf{0.9442} & 0.0200 & 0.9422 & \textbf{0.0156} & 0.9432 & 0.0172 & 0.9422 & 0.0198 \\
& 90.0  & 0.9425 & 0.0208 & 0.9400 & \textbf{0.0168} & 0.9428 & 0.0178 & \textbf{0.9432} & 0.0168 \\
\midrule
\multicolumn{10}{l}{\textbf{Post-prune (no fine-tuning) accuracy (for reference)}} \\
\midrule
& 10.0  & 0.9462 & 0.0108 & 0.9459 & 0.0119 & 0.9464 & 0.0094 & 0.9459 & 0.0100 \\
& 20.0  & 0.9459 & 0.0098 & 0.9457 & 0.0092 & 0.9458 & 0.0090 & 0.9455 & 0.0092 \\
& 30.0  & 0.9457 & 0.0104 & 0.9454 & 0.0090 & 0.9454 & 0.0107 & 0.9455 & 0.0095 \\
& 40.0  & 0.9457 & 0.0098 & 0.9451 & 0.0093 & 0.9438 & 0.0105 & 0.9447 & 0.0108 \\
& 50.0  & 0.9457 & 0.0094 & 0.9453 & 0.0066 & 0.9434 & 0.0082 & 0.9432 & 0.0086 \\
& 60.0  & 0.9446 & 0.0089 & 0.9443 & 0.0058 & 0.9428 & 0.0058 & 0.9424 & 0.0074 \\
& 70.0  & 0.9433 & 0.0106 & 0.9417 & 0.0057 & 0.9428 & 0.0042 & 0.9421 & 0.0038 \\
& 80.0  & 0.9434 & 0.0089 & 0.9395 & 0.0116 & 0.9420 & 0.0043 & 0.9414 & 0.0038 \\
& 90.0  & 0.9400 & 0.0061 & 0.9378 & 0.0176 & 0.9420 & 0.0090 & 0.9416 & 0.0081 \\
\bottomrule
\end{tabular}

\label{tab:bert_headonly_pruning_agnews}
\end{table*}

\end{document}